\documentclass[10pt]{article}

\usepackage[utf8]{inputenc}

\usepackage{epsf}
\usepackage{amsmath}

\allowdisplaybreaks

\usepackage[showframe=false]{geometry}
\usepackage{changepage}

\usepackage{epsfig}
\usepackage{amssymb}

\usepackage{amsthm}
\usepackage{setspace}
\usepackage{cite}
\usepackage{mcite}

\usepackage{algorithmic}  
\usepackage{algorithm}

\usepackage{shadow}
\usepackage{fancybox}
\usepackage{fancyhdr}

\usepackage{color}
\usepackage[usenames,dvipsnames,svgnames,table]{xcolor}

\definecolor{mypurple}{rgb}{.4,.0,.5}

\usepackage{xcolor}

\usepackage{color}

\definecolor{darkgreen}{rgb}{0, 0.4,0}

\definecolor{purplebrown}{rgb}{0.5,0.1,0.6}

\definecolor{ultclupcol}{rgb}{0.1,0.5,0.5}

\definecolor{mytrycolor}{rgb}{0.5,0.7,0.2}

\definecolor{ultclupcola}{rgb}{.5,0,.5}

\newcommand{\bl}[1]{\textcolor{blue}{#1}}

\newcommand{\prp}[1]{\textcolor{purple}{#1}}
\newcommand{\yellow}[1]{\textcolor{yellow}{#1}}

\definecolor{shadebrown}{rgb}{0.1,0.1,0.9}
\definecolor{lightblue}{rgb}{0.2,0,1}

\usepackage{fancybox}
\usepackage{graphicx}
\usepackage{epstopdf}
\usepackage{epsfig}
\usepackage{wrapfig}
\usepackage{subfigure}

\usepackage{xcolor}

\usepackage{tcolorbox}
\tcbuselibrary{skins}

\usepackage{amsmath}
    
    \usepackage{graphicx}
    \usepackage{dcolumn}
    \usepackage{bm}
    \usepackage{amsmath}

\newtcbox{\xmybox}{on line,
arc=7pt,
before upper={\rule[-3pt]{0pt}{10pt}},boxrule=0pt,
boxsep=0pt,left=6pt,right=6pt,top=0pt,bottom=0pt,enhanced, coltext=blue, colback=white!10!yellow}

\newtcbox{\xmyboxa}{on line,
arc=7pt,
before upper={\rule[-3pt]{0pt}{10pt}},boxrule=0pt,
boxsep=0pt,left=6pt,right=6pt,top=0pt,bottom=0pt,enhanced, colback=white!10!yellow}

\newtcbox{\xmyboxb}{on line,
arc=7pt,
before upper={\rule[-3pt]{0pt}{10pt}},boxrule=1pt,colframe=darkgreen!100!blue,
boxsep=0pt,left=6pt,right=6pt,top=0pt,bottom=0pt,enhanced, colback=white!10!yellow}

\newtcbox{\xmyboxc}{on line,
arc=7pt,
before upper={\rule[-3pt]{0pt}{10pt}},boxrule=.7pt,colframe=blue!100!blue,
boxsep=0pt,left=6pt,right=6pt,top=0pt,bottom=0pt,enhanced, coltext=blue, colback=white!10!yellow}

\newtcbox{\xmytboxa}{on line,
arc=7pt,
before upper={\rule[-3pt]{0pt}{10pt}},boxrule=.0pt,colframe=pink!50!yellow,
boxsep=0pt,left=6pt,right=6pt,top=0pt,bottom=0pt,enhanced, coltext=white, colback=blue!40!red}

\newtcbox{\xmytboxb}{on line,
arc=7pt,
before upper={\rule[-3pt]{0pt}{10pt}},boxrule=.0pt,colframe=pink!50!yellow,
boxsep=0pt,left=6pt,right=6pt,top=0pt,bottom=0pt,enhanced, coltext=white, colback=white!40!green}

\usepackage[hyphens]{url}

\usepackage[colorlinks=true,
            linkcolor=black,
            urlcolor=blue,
            citecolor=purple]{hyperref}

\usepackage{breakurl}

\makeatletter
\newcommand\subsubsubsection{\@startsection{paragraph}{4}{\z@}{-2.5ex\@plus -1ex \@minus -.25ex}{1.25ex \@plus .25ex}{\normalfont\normalsize\bfseries}}
\newcommand\subsubsubsubsection{\@startsection{subparagraph}{5}{\z@}{-2.5ex\@plus -1ex \@minus -.25ex}{1.25ex \@plus .25ex}{\normalfont\normalsize\bfseries}}
\makeatother

\def\y{{\bf y}}

\def\y{{\bf y}}

\def\c{{\bf c}}

\def\tr{\mbox{Tr}}

\def\tr{{\rm tr}\,}

\def\diag{{\rm diag}\,}

\def\be{\begin{equation}}
\def\ee{\end{equation}}
\def\ba{\left[\begin{array}}
\def\ea{\end{array}\right]}

\def\y{{\bf y}}

\def\c{{\bf c}}

\def\1{{\bf 1}}

\def\0{{\bf 0}}

\def\vecw{\mbox{vec}}
\def\rankw{\mbox{rank}}
\def\diag{\mbox{diag}}

\def\bU{\bar{U}}

\def\bV{\bar{V}}

\def\mR{{\mathbb R}}
\def\mC{{\mathbb C}}

\def\mP{{\mathbb P}}

\def\calV{{\cal V}}
\def\calU{{\cal U}}

\def\lp{\left (}
\def\rp{\right )}

\newtheorem{theorem}{Theorem}
\newtheorem{corollary}{Corollary}

\newtheorem{lemma}{Lemma}

\begin{document}

\begin{singlespace}

\title {\textbf{Causal Inference ({C}-inf) --- \emph{asymmetric} scenario of typical phase transitions}
}
\author{
\textsc{Agostino Capponi \footnote{e-mail: {\tt ac3827@columbia.edu}}}\quad  \textsc{Mihailo Stojnic \footnote{e-mail: {\tt flatoyer@gmail.com}}} \quad
\\
{Department of Industrial Engineering and Operations Research}\\
{Columbia University, New York, NY 10027, USA}
 }
\date{}
\maketitle

\centerline{{\bf Abstract}} \vspace*{0.1in}

In this paper, we revisit and further explore a mathematically rigorous connection between  \bl{\textbf{\emph{Causal inference (C-inf)}}} and the \bl{\textbf{\emph{Low-rank recovery (LRR)}}} established in  \cite{Cinfidealwc22}. Leveraging the Random duality - Free probability theory (RDT-FPT) connection, we obtain the \emph{exact explicit} typical C-inf \prp{\textbf{\emph{asymmetric}}} phase transitions (PT). We uncover a \bl{\textbf{\emph{doubling low-rankness}}} phenomenon, which means that exactly two times larger low rankness is allowed in asymmetric scenarios compared to the symmetric worst case ones considered in \cite{Cinfidealwc22}. Consequently, the final PT mathematical expressions are as elegant as those obtained in \cite{Cinfidealwc22}, and highlight direct relations between the targeted C-inf matrix low rankness and the time of treatment. Our results have strong implications for applications, where C-inf matrices are not necessarily symmetric.

\vspace*{0.25in} \noindent {\bf Index Terms: Causal inference; Random duality theory; Algorithms; Matrix completion; Sparsity}.

\end{singlespace}

\section{Introduction}
\label{sec:back}

\bl{\textbf{\emph{Causal inference (C-inf)}}} deals with the design of estimation strategies that allow researchers to draw causal conclusions based on data.  The overarching goal is to draw a conclusion regarding the effect of a causal variable, which is typically referred to as the ``treatment'' or the ``intervention'' on some outcome of interest. For example, suppose we want to estimate the causal effect of a drug on deadly cancer progression (vs no exposure to the drug). Then we want to compare metastasis in the patient's body one month after the drug regime has begun versus
metastasis in the absence of  exposure to the drug. The main challenge for causal inference is that we are not generally able to observe both of these states: at the point in time when we are measuring the outcomes, each individual either has had drug exposure or has not.

The problem of estimating the counterfactual, i.e., what would have been the outcome in the absence of a treatement, is central in many disciplines, including economics, health, and social sciences (see, e.g. \cite{RoseRub83,Rub06,ImbRub15,ADHsynth10,DoudImb16,Xucinf17,HerRob10}), machine learning and theoretical computer science (see, e.g. \cite{PearlBar19,PearlCausBook09,PearlSMack18,PearlSurv09}).  Methodological developments to estimate causal effects have been based on experimental or observational data. Experimental research offers the most plausibly unbiased estimates, but experiments are frequently infeasible because they are costly or subject to moral objections. 
Observational data instead are becoming increasing available due to technological advancements in the design of sensor and hardware devices. Our focus is on causal inference in observational studies, and specifically on the design of efficient algorithmic techniques to estimate counterfactuals.	
	
The C-inf approaches can be broadly classified into three categories: 1) the unconfoundedness (see, e.g. \cite{RoseRub83,ImbRub15}); 2) the synthetic control (see, e.g. \cite{ADHsynth10,DoudImb16,Abadsynth19}); and 3) the matrix completion (see, e.g. \cite{ABDIK21, Agarwal2021,KallusNIPS}. Matrix completion methods build upon the foundation works of \cite{	CR09matcomp,Rechtmatcomp11,CPmatcomp10}). Perhaps unexpectedly, all three methods heavily rely on mathematical, statistical, and ultimately algorithmic concepts with very deep roots in information theory. Our work is positioned within the third line of work that mathematically resembles the \bl{\textbf{\emph{matrix completion (MC)}}} problem.

Along the same lines, our work extends significantly the analysis developed in the companion paper  \cite{Cinfidealwc22}. Therein, we obtained the \emph{exact explicit} typical \prp{\textbf{\emph{worst case}}} C-inf phase transitions (PT), and further showed that these phase transitions are achievable by the symmetric targeted C-inf matrices. In the present paper, we consider a generic asymmetric context, to deal with the situation that C-inf matrices are not necessarily always symmetric in real applications. This allows us improving upon the results from \cite{Cinfidealwc22} in certain scenarios. We build further upon the RDT-FPT synergistic mechanisms considered in \cite{Cinfidealwc22}, and precisely characterize the corresponding asymmetric PTs. We also uncover a \bl{\textbf{\emph{doubling low-rankness}}} phenomenon, which means that exactly two times larger low rankness is allowed in asymmetric scenarios compared to the symmetric worst case ones of \cite{Cinfidealwc22}.

\section{Causal inference mathematical setup}
\label{sec:mbcinf}

In this section, we revisit the explicit \prp{\textbf{\emph{causal inference (C-inf) $\leftrightarrow$ matrix completion (MC)}}}  connection, established in \cite{Cinfidealwc22}. Therein, we have discussed the connection between \emph{low rank recovery (LRR)}, \emph{matrix completion (MC)}, and the \emph{causal inference (C-inf)}. Mathematically speaking, one has that the MC is a special case of the LRR and the C-inf is a special case of the MC itself. Consequently, the mathematical models that describe the LRR problems can be used to describe the MC and ultimately the C-inf ones as well. Below we present the C-inf mathematical setup developed through such a connection in \cite{Cinfidealwc22}.

We start with a low rank matrix $X_{sol}\in\mR^{n\times n}$ with the singular value decomposition (SVD)
\begin{eqnarray}\label{eq:mcsvd1}
X=U\Sigma V^T,
\end{eqnarray}
where
\begin{eqnarray}\label{eq:mcsvd2}
\sigma(X)\triangleq\diag(\Sigma) \quad \mbox{and} \quad U^TU=I_{n\times n} \quad \mbox{and} \quad V^TV=I_{n\times n},
\end{eqnarray}
with $I_{n\times n}$ being the $n\times n$ identity matrix and $\diag(\cdot)$ being the operator that creates a column vector of the diagonal elements of its matrix argument. We then define $\ell_p^*(X)$ to be the so-called $\ell_p$ (quasi) norm of $\sigma(X)$ (the vector of the singular values of $X$), i.e.
\begin{eqnarray}\label{eq:mclinsys1b1}
\ell_p^*(X)\triangleq \ell_p(\sigma(X)), p\in\mR_+.
\end{eqnarray}
The following limiting $\ell_p(\cdot)$ connections are important as well
\begin{eqnarray}\label{eq:mclinsys1c}
\ell_0^*(X_{sol})\triangleq \ell_0(\sigma(X_{sol}))=\|\sigma(X_{sol})\|_0=\lim_{p\longrightarrow 0}\|\sigma(X_{sol})\|_p= \lim_{p\longrightarrow 0}\ell_p(\sigma(X_{sol}))=\lim_{p\longrightarrow 0}\ell_p^*(X_{sol}).
\end{eqnarray}
Moreover, we also define the so-called block masking matrix $M$ as  (see Figure \ref{fig:Mnlockcinf} as well)
\begin{center}
\tcbset{beamer,lower separated=false, fonttitle=\bfseries,
coltext=black , colback=yellow!70!orange!40!white ,title style={left color=black, right color=red!70!orange!30!white}}
\begin{tcolorbox}[title=$M$ matrix in block causal inference (C-inf):, width=6.2in]
\vspace{-.15in}
\begin{equation}\label{eq:cinfanl2a}
  M\triangleq M^{(l)}\triangleq \1_{n\times 1}\1_{n\times 1}^T- I^{(l)}(I^{(l)})^T   \1_{n\times 1}\1_{n\times 1}^T   I^{(l)}(I^{(l)})^T \quad \mbox{and} \quad I^{(l)} \triangleq \begin{bmatrix}
        \0_{l\times (n-l)} \\
        I_{(n-l)\times (n-l)}
      \end{bmatrix}.
\end{equation}
\vspace{-.25in}
\end{tcolorbox}
\end{center}
\begin{figure}[htb]
\centering
\centerline{\epsfig{figure=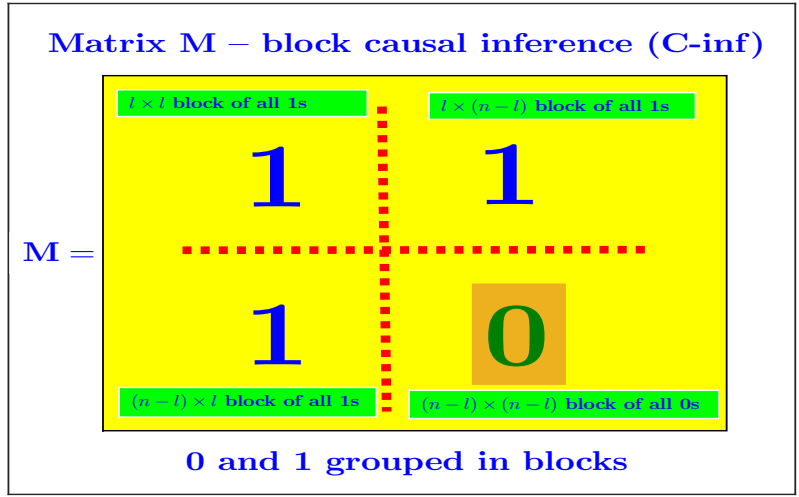,width=13.5cm,height=8cm}}
\caption{Matrix $M\triangleq M^{(l)}$ -- block causal inference (\bl{\textbf{C-inf}}) setup}
\label{fig:Mnlockcinf}
\end{figure}
One then has the following two optimization problems that are at the heart of the \prp{\textbf{\emph{C-inf $\leftrightarrow$ MC}}} connection
\begin{center}
\tcbset{beamer,lower separated=false, fonttitle=\bfseries,width=3.4in, coltext=black ,
colback=yellow!70!orange!40!white,title style={left color=cyan!40!black!80!purple, right color=red!60!yellow!40!orange!80!white},
width=(\linewidth-4pt)/4, equal height group=AT,before=,after=\hfill,fonttitle=\bfseries}
\begin{tcolorbox}[title=$\ell_0^*$-minimization (C-inf -- \yellow{MMT}), width=3.1in]
\vspace{-.15in}
\begin{eqnarray}\label{eq:genmcl0posmmt}
\min_{X} & & \ell_0^*(X) \nonumber \\
\hspace{-.0in} \mbox{subject to} & & Y=M\circ X.\hspace{.4in}
\end{eqnarray}
\vspace{-.32in}
\end{tcolorbox}
\begin{tcolorbox}[title=$\ell_1^*$-minimization (C-inf -- \yellow{MMT}), width=3.2in]
\vspace{-.15in}
\begin{eqnarray}\label{eq:genmcl1posmmt}
\min_{X} & & \ell_1^*(X) \nonumber \\
 \mbox{subject to} & & Y=M\circ X, \hspace{.5in}
\end{eqnarray}
\vspace{-.32in}
\end{tcolorbox}
\end{center}
where $\circ$ stands for the component-wise multiplication. Namely, keeping in mind that $\ell_0^*(X)$ effectively counts the number of the nonzero singular values of $X$, the optimization problem in (\ref{eq:genmcl0posmmt}) is exactly the recovery of the C-inf targeted low rank matrix $X$ from the linear observations $Y$ obtained through a masking via $M$. Moreover, the problem in (\ref{eq:genmcl0posmmt}) (with a generic $M$) is a standard matrix completion setup which on the other hand is a special case of the LRR problems (expressed in the ``masking matrix terminology" (MMT)). On the other hand, the optimization problem in (\ref{eq:genmcl1posmmt}) is the tightest convex relaxation heuristic typically utilized in the matrix completion literature for solving NP-hard problem (\ref{eq:genmcl0posmmt}). For more on the origin of these two problems and their connection within the LRR and MC context we refer to the introductory LRR/MC papers \cite{RFPrank,CR09matcomp,SAT05}. More on their importance and different related algorithmic considerations can be found in many papers that followed (see, e.g. \cite{CT10matcomp,KMO10matcomp,KMO10matcomp1,Klopp14matcomp,KLT11matcomp,NW11matcomp,NW12matcomp,RT11matcomp,MHT10}).

Here though, we would particularly like to point out reference \cite{ABDIK21} where the very same C-inf context was considered and the very same \prp{\textbf{\emph{C-inf $\leftrightarrow$ MC}}} connection recognized. Considerations from \cite{ABDIK21} are in fact especially convenient to properly understand in what C-inf contexts the block structure of the matrix $M$ might appear. To see that one can connect it to the so-called counterfactuals and the units/treatments terminology employed in \cite{ABDIK21}.

First we note that $M$ can be alternatively defined as
\begin{equation}
M_{i,j}=\begin{cases}
          1, & \mbox{$(i,j)$-th element of $X_{sol}$ is observed}  \\
          0, & \mbox{otherwise}.
        \end{cases}  \label{eq:mc2}
\end{equation}
It is then rather clear that ones in $M$ allow reading out the corresponding elements of $X_{sol}$ while zeros block (mask) them. Then the context of \cite{ABDIK21} is roughly as follows. One first assumes that the matrix $X$ contains observations about a certain set of, say, $n$ units {(e.g. individuals, subpopulations, and geographic regions)} over a period of say, $n$, time instances. After that the rows of $X$ are allocated to the units and the columns to the time instances and one would like to estimate the effects that a certain treatment may have on the treated units. A subset of the units (say those that correspond to the rows $i>l$) is then at time $l$ exposed to an irreversible treatment. {Examples of treatments include health therapies, socio-economic policies, and taxes.} To ensure an appropriate assessment of the resulting treatment effects, in addition to having the values of $X$ after the treatment, one would need to have the access to the so-called \bl{\textbf{\emph{counterfactuals}}} -- the values of the treated units -- had the treatment not been applied. Relating back to the  matrix completion terminology, one would basically need to estimate (a presumably low rank) $X$ while not having access to its portion covered by the block-mask $M=M^{(l)}$. In other words, one would need to solve (\ref{eq:genmcl0posmmt}) with $M=M^{(l)}$.

The above describes the C-inf via counterfactuals and the underlying role of matrix $M$. Moreover, if one views things in the time domain, i.e. if the columns of $M$ represent time axis, then the observations in ceratin rows will not be available after a fixed point in time. In the block scenario this point is fixed across the affected rows. However, it does not necessarily need to be fixed (for more in this direction we refer to \cite{ADHsynth10} (in particular, the California tobacco example), \cite{XPlatfac10} (in particular, the latent factor modeling in the context of the simultaneous/staggered treatment adoption), and to \cite{AthImb18,AthSte02,ShaTou19} (in particular, the health care applications) as excellent references for understanding the need of various C-inf scenarios). As this and \cite{Cinfidealwc22} are the introductory papers, where we present the overall methodology, we selected the block causal inference scenario as probably the most representative and well-known one. In some of our companion papers we will show how the methodology that we are introducing here can be utilized to handle other C-inf scenarios as well.

\section{$\ell_0^*-\ell_1^*$ equivalence}
\label{sec:cinfreleqv}

As mentioned earlier, solving the generic LRR (and consequently the C-inf as its a special case) might be difficult due to a highly non-convex objective function in (\ref{eq:genmcl0posmmt}). Various heuristics can be employed depending on the practical scenarios that one can face. In the mathematically most challenging so-called linear regime, the above mentioned $\ell_1^*$-minimization relaxation heuristic (often called nuclear norm minimization) is typically viewed as the best known provably polynomial one. We adopt the same view in what follows and take it as a current benchmark for the algorithmic handling of the C-inf. As mentioned above, a rather remarkable feature of this heuristic is that sometimes it can actually solve the underlying problems exactly. When that happens we say that the following $\ell_0^*-\ell_1^*$-equivalence phenomenon occurs.

\begin{center}
\tcbset{beamer,lower separated=false, fonttitle=\bfseries,
coltext=black , colback=yellow!70!orange!40!white ,title style={left color=black, right color=red!70!orange!30!white}}
\begin{tcolorbox}[title=$\ell_0^*-\ell_1^*$-equivalence (C-inf): \yellow{$\ell_0^*\Longleftrightarrow\ell_1^*$}, width=6.43in]
\vspace{-.0in}
Let $X_{sol}$ be the solution of (\ref{eq:genmcl0posmmt}) and let $\hat{X}$ be a solution of (\ref{eq:genmcl1posmmt}) and set
\begin{eqnarray*}\label{eq:genmcl1poseqvrmse}
\mathbf{RMSE}\triangleq\|\vecw(\hat{X})-\vecw(X_{sol})\|_2.
\end{eqnarray*}
\vspace{-.3in}
\begin{eqnarray}\label{eq:genmcl1poseqv}
\mbox{If and only if } \prp{(\hat{X}=X_{sol} \mbox{ and } \mathbf{RMSE}=0)} \quad \mbox{then} \quad \prp{(\ell_0^{*}-\mbox{minimization} \Longleftrightarrow \ell_1^{*}-\mbox{minimization})}.\quad
\end{eqnarray}
\vspace{-.3in}
\end{tcolorbox}
\end{center}

The above basically means that when the $\ell_0^*-\ell_1^*$-equivalence happens the optimization problems in (\ref{eq:genmcl0posmmt}) and (\ref{eq:genmcl1posmmt}) are equivalent and as such replaceable by each other. We denote such a phenomenon as $\ell_0^*\Longleftrightarrow\ell_1^*$. That would, of course, be an ideal scenario where it would be basically possible to replace the non-convex optimization problem with the convex one without losing anything in terms of the accuracy of the obtained solutions. Since the mere existence of such a phenomenon is rather remarkable we will in this paper be interested in uncovering the underlying intricacies that enable for it ro happen. Moreover, as it will turn out that its occurrence is not an anomaly but rather a consequence of a generic property, we will then raise the bar accordingly and attempt to provide not only the proof of its existence but also its a complete analytical characterization. This will include a full characterization as to how often and in what scenarios it might happen. To do so we will combine the Random Duality Theory (RDT) tools from \cite{StojnicCSetam09,StojnicUpper10,StojnicICASSP10var,StojnicICASSP10knownsupp,StojnicICASSP09,StojnicISIT2010binary,StojnicRegRndDlt10,StojnicGenLasso10} and several advanced sophisticated probabilistic concepts that we will introduce along the way in the sections that follow below.

We start with some algebraic $\ell_0^*-\ell_1^*$-equivalence preliminaries which are borrowed from the RDT. The first one is a generic LRR $\ell_0^*-\ell_1^*$-equivalence result (the result is basically an adaptation of the corresponding CS equivalence condition from \cite{StojnicCSetam09,StojnicUpper10,StojnicICASSP09} (similar adaptation can also be found in \cite{OH10})).

\begin{theorem}(\cite{Cinfidealwc22} \textbf{\bl{$\ell_0^*-\ell_1^*$-equivalence condition (LRR)}} -- \textbf{general}  $X$)
Consider a $\bU\in\mR^{n\times k}$ such that $\bU^T\bU=I_{k\times k}$ and a $\bV\in\mR^{n\times k}$ such that $\bV^T\bV=I_{k\times k}$ and a  $\rankw-k$  matrix $X_{sol}=X\in\mR^{n\times n}$  with all of its columns belonging to the span of $\bU$ and all of its rows belonging to the span of $\bV^T$. Also, let the orthogonal spans $\bU^{\perp}\in\mR^{n\times (n-k)}$ and $\bV^{\perp}\in\mR^{n\times (n-k)}$ be such that $U\triangleq \begin{bmatrix}
    \bU & \bU^{\perp}
   \end{bmatrix}$ and $V\triangleq \begin{bmatrix}
    \bV & \bV^{\perp}
   \end{bmatrix}$ and
\begin{equation}\label{eq:cinfthm0}
U^TU\triangleq \begin{bmatrix}
    \bU & \bU^{\perp}
   \end{bmatrix}^T\begin{bmatrix}
    \bU & \bU^{\perp}
   \end{bmatrix}=I_{n\times n} \quad \mbox{and} \quad
 V^TV \triangleq\begin{bmatrix}
    \bV & \bV^{\perp}
   \end{bmatrix}^T\begin{bmatrix}
    \bV & \bV^{\perp}
   \end{bmatrix}=I_{n\times n}.
 \end{equation}
For a given matrix $A\in\mR^{m\times n^2}$ ($m\leq n^2$) assume that $\y=A\vecw(X)=A\vecw(X_{sol})\in \mR^m$ and let $\hat{X}$ be the solution of (\ref{eq:genmcl1posmmt}). If
\begin{equation}
(\forall W\in \mR^{n\times n} | A\vecw(W)=\0_{m\times 1},W\neq \0_{n\times n}) \quad  -\tr(\bU^TW\bV)< \ell_1^*((\bU^{\perp})^TW\bV^{\perp}),
\label{eq:cinfthm1}
\end{equation}
then
\begin{equation}
\ell_0^*\Longleftrightarrow \ell_1^* \quad \mbox{and}\quad  \textbf{\emph{RMSE}}=\|\mbox{\emph{vec}}(\hat{X})-\mbox{\emph{vec}}(X_{sol})\|_2=0,\label{eq:cinfthm1a}
\end{equation}
and the solutions of (\ref{eq:genmcl0posmmt}) and (\ref{eq:genmcl1posmmt})  coincide. Moreover, if
\begin{equation}
(\exists  W\in \mR^{n\times n} | A\vecw(W)=\0_{m\times 1},W\neq \0_{n\times n}) \quad  -\tr(\bU^TW\bV)\geq \ell_1^*((\bU^{\perp})^TW\bV^{\perp}),
\label{eq:cinfthm2}
\end{equation}
then there is an $X$ from the above set of matrices with columns belonging to the span of $\bU$ and rows belonging to the span of $\bV$ such that the solutions of (\ref{eq:genmcl0posmmt})  and (\ref{eq:genmcl1posmmt})  are different.
\label{thm:cinfthm1}
\end{theorem}
\begin{proof}
  The proof is a trivial adaptation of the proof for symmetric matrices given in Appendix \ref{sec:appA}.
\end{proof}

Continuing further in the spirit of the RDT the following corollary is a matrix completion specific variant of the above theorem.
\begin{corollary}(\cite{Cinfidealwc22} \textbf{\bl{$\ell_0^*-\ell_1^*$-equivalence condition via masking matrix (MC/C-inf)}} -- \textbf{general}  $X$)
Assume the setup of Theorem \ref{thm:cinfthm1} with $X_{sol}$ being the unique solution of (\ref{eq:genmcl0posmmt}). Let the masking matrix $M\in\mR^{n\times n}$ have $m$ ones and $(n^2-m)$ zeros and let $A$ be generated via $M$, i.e. let $A$ be the matrix obtained after removing all the zero rows from $\diag^{-1}(\vecw(M))I_{n^2\times n^2}$. If and only if
\begin{equation}
\min_{W,W^TW=1,M\circ W=\0_{n\times n}}  \tr(\bU^TW\bV)+\ell_1^*((\bU^{\perp})^TW\bV^{\perp})\geq 0,
\label{eq:cinfcor1}
\end{equation}
then
\begin{equation}
\ell_0^*\Longleftrightarrow \ell_1^* \quad \mbox{and}\quad  \textbf{\emph{RMSE}}=\|\mbox{\emph{vec}}(\hat{X})-\mbox{\emph{vec}}(X_{sol})\|_2=0,\label{eq:cinfcor1a}
\end{equation}
and the solutions of (\ref{eq:genmcl0posmmt})  and (\ref{eq:genmcl1posmmt}) coincide.
 \label{cor:cinfcor1}
\end{corollary}

Finally, the following spectral oriented corollary was proven in \cite{Cinfidealwc22} as well.

\begin{corollary}(\cite{Cinfidealwc22} \textbf{\bl{$\ell_0^*-\ell_1^*$-equivalence condition via mask-modified bases spectra (C-inf)}} -- \textbf{general}  $X$)
Assume the setup of Theorem \ref{thm:cinfthm1} with $k\leq l$. Let $M\triangleq M^{(l)}\in\mR^{n\times n}$ and $I^{(l)}\in\mR^{n\times (n-l)}$ be as defined in (\ref{eq:cinfanl2a}). Set
\begin{eqnarray}
\Lambda_V & \triangleq & ((I^{(l)})^T \bV^{\perp})^{-1} (I^{(l)})^T \bV \nonumber \\
\Lambda_U & \triangleq & ((I^{(l)})^T \bU^{\perp})^{-1} (I^{(l)})^T \bU \nonumber \\
 Q & = &   \lp (I^{(l)})^T\bV^{\perp}(\bV^{\perp})^T I^{(l)}\rp^{-1} -I \nonumber \\
  Q_1^{\perp} & = &   \lp (I^{(l)})^T\bU^{\perp}(\bU^{\perp})^T I^{(l)}\rp^{-1} -I. \label{eq:cinfanl9}
\end{eqnarray}
\begin{center}
 \begin{tcolorbox}[beamer,title=\textbf{C-inf \yellow{perfectly succeeds:  $\ell_0^*\Longleftrightarrow \ell_1^* \quad \mbox{and}\quad  \textbf{\emph{RMSE}}=\|\mbox{\emph{vec}}(\hat{X})-\mbox{\emph{vec}}(X_{sol})\|_2=0$}},lower separated=false, colback=yellow!95!green!40!white,
colframe=red!75!blue!60!black,fonttitle=\bfseries,width=6in]
\begin{equation}
 \mbox{If and only if} \quad \lambda_{max}(\Lambda_V^T\Lambda_V\Lambda_U^T\Lambda_U)\leq 1.\label{eq:cinfcor3eq1}
\end{equation}
 \end{tcolorbox}
\end{center}
Moreover, if
\begin{eqnarray}
\lambda_{max} \lp Q\rp
\lambda_{max} \lp Q_1^{\perp}\rp \leq 1,\label{eq:cinfcor3eq2}
\end{eqnarray}
then again $\ell_0^*\Longleftrightarrow \ell_1^*$ and $\textbf{\emph{RMSE}}=\|\mbox{\emph{vec}}(\hat{X}-\mbox{\emph{vec}}(X_{sol})\|_2=0$ and the C-inf perfectly succeeds as well.
\label{cor:cinfcor3}
\end{corollary}

Since we will be working in the mathematically most challenging large $n$ linear regime, we find it useful to introduce the following large dimensional scalings
\begin{eqnarray}
\beta\triangleq \lim_{n\rightarrow \infty}\frac{k}{n}\quad \mbox{and} \quad \eta\triangleq \lim_{n\rightarrow \infty}\frac{l}{n} \quad \mbox{and} \quad \alpha\triangleq \lim_{n\rightarrow \infty}\frac{m}{n^2}=\lim_{n\rightarrow \infty}\frac{n^2-(n-l)^2}{n^2}=1-(1-\eta)^2. \label{eq:typwcanl21}
\end{eqnarray}

The key highlight result of \cite{Cinfidealwc22} is the following theorem obtained through an analysis that relied on the above corollary and a combination of the Random duality theory (RDT) and Free probability theory (FPT). It basically establishes the \emph{worst case} phase-transition (PT) that $\ell_1^*$, tightest convex relaxation heuristic, exhibits when used for solving C-inf in a \emph{typical} statistical scenario.

\begin{theorem}(\textbf{\bl{$\ell_1^*$ -- phase transition -- C-inf (typical \prp{\underline{worst case}})}})
  Consider a rank-$k$ matrix  $X_{sol}=X\in\mR^{n\times n}$ with the Haar distributed (\emph{not necessarily independent}) bases of its orthogonal row and column spans $\bU^{\perp}\in\mR^{n\times (n-k)}$ and $\bV^{\perp}\in\mR^{n\times (n-k)}$ ($X_{sol}^T\bU^{\perp}=X_{sol}\bV^{\perp}=\0_{n\times (n-k)}$). Let $M\triangleq M^{(l)}\in\mR^{n\times n}$ be as defined in (\ref{eq:cinfanl2a}). Assume a large $n$ linear regime with $\beta\triangleq\lim_{n\rightarrow\infty}\frac{k}{n}$ and $\eta\triangleq\lim_{n\rightarrow\infty}\frac{l}{n}$ and let $\beta_{wc}$ and $\eta$ satisfy the  following
\begin{center}
 \begin{tcolorbox}[beamer,title=\textbf{C-inf $\ell_1^*$ \yellow{worst case} phase transition (PT)} characterization,lower separated=false, colback=yellow!95!green!40!white,
colframe=red!75!blue!60!black,fonttitle=\bfseries,width=5in]
  \begin{equation}\label{eq:typwcthm1eq1}
    \xi_{\eta}^{(wc)}(\beta) \triangleq \beta-\frac{1}{2}+\sqrt{\eta-\eta^2}=0.
  \end{equation}
 \end{tcolorbox}
\end{center}
 \noindent \textbf{If and only if} $\beta\leq \beta_{wc}$
    \begin{equation}\label{eq:typwcthm1eq2}
   \lim_{n\rightarrow\infty} \mP(\ell_0^*\Longleftrightarrow \ell_1^*)=\ \lim_{n\rightarrow\infty} \mP(\mathbf{RMSE}=0)=1,
  \end{equation}
and the solutions of (\ref{eq:genmcl0posmmt}) and (\ref{eq:genmcl1posmmt}) coincide with overwhelming probability.
  \label{thm:typwcthm1}
\end{theorem}

The results obtained based on the above theorem are shown in Figure \ref{fig:cinfspectypwcPTbetaeta}, where one can clearly see that the phase transition curve splits the entire $(\beta,\eta)$ region into two subregions: 1) the first one (below (or to the right of) the curve) where the $\ell_0^*-\ell_1^*$-equivalence phenomenon occurs; and 2) the second one (above (or to the left of) the curve) where the $\ell_0^*-\ell_1^*$-equivalence is lacking. This means that one can recover $X_{sol}$ masked by $M$ as in (\ref{eq:genmcl0posmmt}) via the $\ell_1^*$ heuristic from (\ref{eq:genmcl1posmmt}) with the residual mean square error ($\mathbf{RMSE}$) equal to zero. In other words, for the system parameters $(\beta,\eta)$ that belong to the subregion below the curve one has a perfect recovery with $X_{sol}$ and $\hat{X}$ (the respective solutions of (\ref{eq:genmcl0posmmt}) and (\ref{eq:genmcl1posmmt})) being equal to each other and consequently with $\mathbf{RMSE}=\|\vecw(\hat{X})-\vecw(X_{sol})\|_2=0$. On the other hand, in the subregion above the curve, the $\ell_1^*$ heuristic fails and one can even find an $X_{sol}$ for which $\mathbf{RMSE}\rightarrow\infty$.

\begin{figure}[htb]
\centering
\centerline{\epsfig{figure=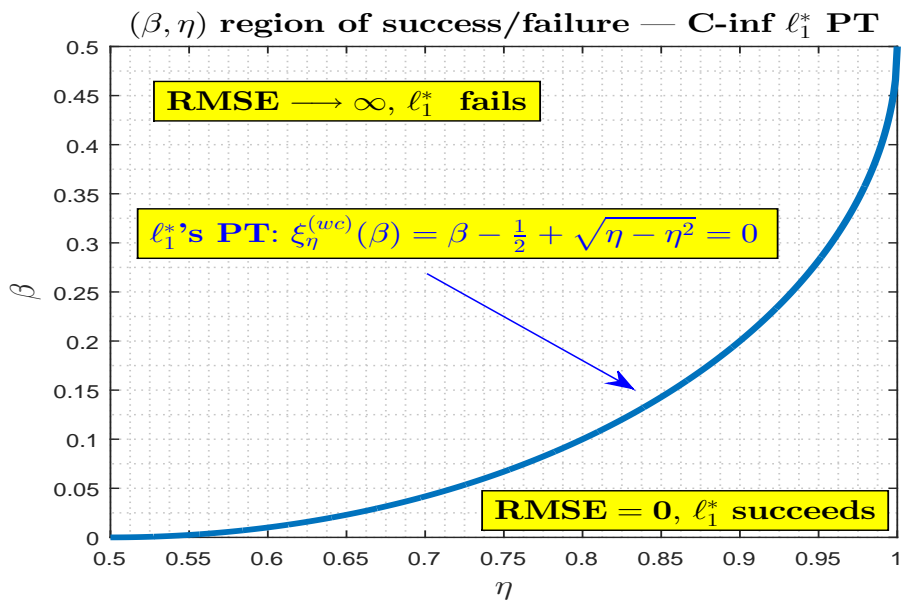,width=13.5cm,height=8cm}}
\caption{Causal inference (\bl{\textbf{C-inf}}) -- typical \emph{\textbf{worst case}} $\ell_1^*$ phase transition}
\label{fig:cinfspectypwcPTbetaeta}
\end{figure}

\section{Analysis of the $\ell_0^*-\ell_1^*$-equivalence -- typical \prp{asymmetric scenario}}
\label{sec:typacanl}

In this section we consider when the conditions given in Corollary \ref{cor:cinfcor3} are met. As in \cite{Cinfidealwc22}, we will be working in a ``\emph{typical}" statistical context. On the other hand, differently from \cite{Cinfidealwc22}, instead of focusing on the \prp{\textbf{\emph{worst case}}} (symmetric) scenario we here consider a typical \prp{\textbf{\emph{asymmetric scenario}}} setup. Practically this means two things: 1) as in \cite{Cinfidealwc22}, both $\bV$ and $\bU$ will be assumed as Haar distributed; and 2) differently from Theorem \ref{thm:typwcthm1} and \cite{Cinfidealwc22}, $\bV$ and $\bU$ will now be assumed as independent of each other. In a way one can view the worst case scenario from \cite{Cinfidealwc22} as an extreme where $\bV$ and $\bU$ are ``\emph{not independent at all}" (or, in other words, equal to each other). Along similar lines, one can then view the scenario that we will consider here as another extreme where $\bV$ and $\bU$ are ``\emph{not dependent at all}" (or, in other words, completely independent). In situations where no particular structure of a low rank nonsymmetric $X_{sol}$ is favored over any other this one would naturally be a most reasonable choice. In other words, it is not only an extreme case, but actually the one that typically might most faithfully describe the performance of the underlying C-inf heuristics.

\subsection{Free probability theory \bl{(FPT)} -- preliminaries}
\label{sec:fptprel}

Below we provide a short preview of the most basic FPT concepts needed for our analysis (we refer to our companion paper \cite{Cinfidealwc22} for a more detail treatment). As is by now well known, the work od Dan Voiculescu on group theories (see, e.g. \cite{Voic86,Voic87,Voic91}) established the foundations of the FPT. As the practical importance of FPT became immediately evident a substantial interest for further studying was generated and, in the years that followed, quite a few nice results appeared that made the whole theory more approachable and ultimately presentable in an easily understandable way. Along the same lines, we follow into the footsteps of \cite{Cinfidealwc22}, leave all the abstractions out and focus on the FPT's key practically applicable components (for further details see also, e.g. \cite{Voic86,Voic87,Voic91,NicaSpeich06,Speich14,Haag97,TulVer04}).

\subsubsection{Basics of FPT -- random matrix variables}
\label{sec:fptprelmatrices}

We assume large $n$ linear regime and consider two symmetric matrices $A=A^T\in\mR^{n\times n}$ and $B=B^T\in\mR^{n\times n}$ with Haar distributed eigenspaces. We also assume that their individual respective spectral laws are $f_A(\cdot)$ and $f_B(\cdot)$. Three different transforms of these spectral densities will be needed. We start with the first one, the so-called Stieltjes (or G) transform
\begin{eqnarray}
    G(z) & \triangleq & \int_{I_f} \frac{f(x)}{z-x} dx, \quad z\in\mC\setminus I_f,  \label{eq:typwcanl6}
\end{eqnarray}
where $I_f$ is the domain of $f(\cdot)$. The following inverse relation is also well known
\begin{eqnarray}
    f(x) =  \lim_{\epsilon\rightarrow 0^+} \frac{G(x-i\epsilon)-G(x+i\epsilon)}{2i\pi}
    \quad \mbox{or} \quad    f(x) =  -\lim_{\epsilon\rightarrow 0^+} \frac{\mbox{imag}(G(x+i\epsilon))}{\pi}.   \label{eq:typwcanl7}
\end{eqnarray}
For the above to hold it makes things easier to implicitly assume that $f(x)$ is continuous. We will, however, utilize it even in discrete (or semi-discrete) scenarios since the obvious asymptotic translation to continuity would make it fully rigorous. A bit later though, when we see some concrete examples where things of this nature may appear, we will say a few more words and explain more thoroughly what exactly can
be discrete and how one can deal with such a discreteness. In the meantime we proceed with general principles not necessarily worrying about all the underlying technicalities that may appear in scenarios deviating from the typically seen ones and potentially requiring additional separate addressing. To that end we continue by considering the $R(\cdot)$- and $S(\cdot)$-transforms that satisfy the following
\begin{eqnarray}
R(G(z))+\frac{1}{G(z)}=z,  \label{eq:typwcanl8}
\end{eqnarray}
and
\begin{eqnarray}
S(z)=\frac{1}{R(zS(z))} \quad \mbox{and}\quad R(z)=\frac{1}{S(zR(z))}.\label{eq:typwcanl9}
\end{eqnarray}
Let $f_A(\cdot)$ and $f_B(\cdot)$ be the spectral distributions of $A$ and $B$ and let $R_A(z)$/$S_A(z)$ and $R_B(z)$/$S_B(z)$ be their associated $R(\cdot)$-/$S(\cdot)$-transforms. One then has the following
\begin{center}
 \begin{tcolorbox}[beamer,title=\textbf{Key Voiculescu's FPT concepts \cite{Voic86,Voic87}:},lower separated=false, colback=yellow!95!green!40!white,
colframe=green!45!blue!60!black,coltext=black,fonttitle=\bfseries,width=5in]
\vspace{-.1in}\begin{eqnarray}
\begin{array}{r c l l r c l}
C & = & A+B & \quad \Longrightarrow \quad $ $ & R_C(z) & = & R_A(z)+R_B(z)\\
C & = & AB  & \quad \Longrightarrow \quad $ $ & S_C(z) & = & S_A(z)S_B(z).
\end{array}\label{eq:typwcanl10}
\end{eqnarray}
 \end{tcolorbox}
\end{center}
\noindent Now it is relatively easy to see that (\ref{eq:typwcanl6})-(\ref{eq:typwcanl10}) are sufficient to determine the spectral distribution of the sum or the product of two independent matrices with given spectral densities and the Haar distributed bases of eigenspaces. The above is of course a generic principle. It can be applied pretty much always as long as one has access to the statistics of the underlying matrices $A$ and $B$. In the following section we will raise the bar a bit higher and show that in the C-inf context one can use all of the above in such a manner that eventually all the quantities of interest are explicitly determined.

\subsubsection{Spectral preliminaries}
\label{sec:fpteqvspecwc}

We start by recalling on $Q$ from (\ref{eq:cinfanl9}) and introducing $Q_1$
\begin{eqnarray}
Q_1 & \triangleq &  \Lambda_V^T\Lambda_V \nonumber \\
Q   & \triangleq &   \lp (I^{(l)})^T\bV^{\perp}(\bV^{\perp})^T I^{(l)}\rp^{-1} -I \nonumber \\
Sp(Q_1) & \Longleftrightarrow_{\setminus 0} & Sp(Q), \label{eq:typwcanl13}
\end{eqnarray}
where $Sp(\cdot)$ stands for the spectrum of the matrix argument and $\Longleftrightarrow_{\setminus 0}$ means the equivalence of the parts of the spectra outside the zero eigenvalues. It is rather obvious that it will then be sufficient to handle the spectrum of
\begin{eqnarray}
D & \triangleq &  (I^{(l)})^T\bV^{\perp}(\bV^{\perp})^T I^{(l)}. \label{eq:typwcanl14}
\end{eqnarray}
Consider Haar distributed $\bU_D^{\perp}\in\mR^{n\times (n-l)}$ with $(\bU_D^{\perp})^T\bU_D^{\perp}=I_{(n-l)\times (n-l)}$ and let
\begin{eqnarray}
U_D & = &  \begin{bmatrix}
             \bU_D & \bU_D^{\perp}
           \end{bmatrix} \quad \mbox{with} \quad U_D^TU_D=I_{n\times n}. \label{eq:typwcanl15}
\end{eqnarray}
Also, we assume that $\bU_D^{\perp}$ (and $U_D$) are independent of $\bV^{\perp}$. After setting
\begin{eqnarray}
\bar{D} & \triangleq &  (I^{(l)})^TU_D^T\bV^{\perp}(\bV^{\perp})^T U_D I^{(l)}, \label{eq:typwcanl16}
\end{eqnarray}
we have that the spectra of $D$ and $\bar{D}$ are statistically identical, i.e.
\begin{eqnarray}
Sp(D) \triangleq  Sp((I^{(l)})^T\bV^{\perp}(\bV^{\perp})^T I^{(l)}) \Longleftrightarrow_\mP
  Sp((I^{(l)})^TU_D^T\bV^{\perp}(\bV^{\perp})^T U_D I^{(l)}) \triangleq Sp(\bar{D}), \label{eq:typwcanl17}
\end{eqnarray}
where $\Longleftrightarrow_\mP$ stands for the statistical/probabilistic equivalence. Two facts enable the above statistical identity: 1) the spectrum of the projector $\bV^{\perp}(\bV^{\perp})^T$ does not change under pre- and post-unitary multiplications on both sides; and 2) the Haar structure of $\bV^{\perp}$ remains preserved. Modulo zero eigenvalues, we then further have
\begin{eqnarray}
  Sp((I^{(l)})^TU_D^T\bV^{\perp}(\bV^{\perp})^T U_D I^{(l)}) \Longleftrightarrow_{\mP\setminus 0}
  Sp(\bV^{\perp}(\bV^{\perp})^T U_D I^{(l)}(I^{(l)})^TU_D^T) \Longleftrightarrow
  Sp(\bV^{\perp}(\bV^{\perp})^T \bU_D^{\perp}(\bU_D^{\perp})^T), \label{eq:typwcanl18}
\end{eqnarray}
where, similarly as above, $\Longleftrightarrow_{\mP\setminus 0}$ stands for the statistical/probabilistic equivalence in the part of the spectrum outside the zero eignevalues (introduced due to the non-square underlying matrices). Clearly, the key object of our interest below will be
\begin{eqnarray}
\tilde{D} & \triangleq & \bV^{\perp}(\bV^{\perp})^T \bU_D^{\perp}(\bU_D^{\perp})^T, \label{eq:typwcanl19}
\end{eqnarray}
where both $\bV^{\perp}$ and $\bU_D^{\perp}$ are Haar distributed and independent of each other. After setting
\begin{eqnarray}
{\cal V} & \triangleq & \bV^{\perp}(\bV^{\perp})^T \nonumber \\
{\cal U} & \triangleq & \bU_D^{\perp}(\bU_D^{\perp})^T, \label{eq:typwcanl20}
\end{eqnarray}
we easily have from (\ref{eq:typwcanl19})
\begin{eqnarray}
\tilde{D} & \triangleq & {\cal V}{\cal U}. \label{eq:typwcanl20a}
\end{eqnarray}
The following lemma proven in \cite{Cinfidealwc22} characterizes the $G$-transform of $\tilde{D}$..
\begin{lemma}(\cite{Cinfidealwc22})
  Let $\bV^{\perp}\in\mR^{n\times (n-k)}$ and $\bU_D^{\perp}\in\mR^{n\times (n-k)}$ be Haar distributed unitary bases of $(n-k)$-dimensional subspaces of $\mR^n$. Let $\calV$ and $\calU$ be as in (\ref{eq:typwcanl20}) and $\tilde{D}$ as in (\ref{eq:typwcanl20a}), i.e.
   \begin{eqnarray}
  {\cal V} & \triangleq & \bV^{\perp}(\bV^{\perp})^T \nonumber \\
    {\cal U} & \triangleq & \bU_D^{\perp}(\bU_D^{\perp})^T\nonumber \\
     \tilde{D} & \triangleq & {\cal V}{\cal U}.\label{eq:typwclemma1aeq1}
\end{eqnarray}
In the large $n$ linear regime, with $\beta\triangleq\lim_{n\rightarrow\infty}\frac{k}{n}$, the $G$-transform of the spectral density of $\tilde{D}$, $f_{\tilde{D}}(\cdot)$, is
\begin{eqnarray}
G_{\tilde{D}}^{\pm}(z)=\frac{z-(\beta+\eta)\pm\sqrt{(z-(\beta+\eta))^2+4\beta\eta(z-1)}}{2(z^2-z)}.\label{eq:typwclemma1aeq2}
\end{eqnarray}\label{lemma:typwclemma1a}
\end{lemma}

\subsection{Asymmetric scenario -- \bl{FPT} analysis of the $\ell_0^*-\ell_1^*$-equivalence}
\label{sec:fpteqvac}

As in \cite{Cinfidealwc22}, we will again rely on the free probability theory. This time though things will be a bit more complicated as
we will be determining, so to say, the ``joint spectrum"  of $\lambda_V^T\lambda_V\lambda_U^T\lambda_U$. In other words, based on Corollary \ref{cor:cinfcor3} and (\ref{eq:cinfcor3eq1}), we have
 \begin{eqnarray}
\ell_0^*-\ell_1^*--\mbox{equivalence} \quad \Longleftrightarrow \quad \lambda_{max}(\lambda_V^T\lambda_V\lambda_U^T\lambda_U)\leq 1, \label{eq:typacanl0a0}
\end{eqnarray}
and consequently determining the upper edge of the ``joint spectrum"  of $\lambda_V^T\lambda_V\lambda_U^T\lambda_U$ would be then sufficient to establish $\ell_0^*-\ell_1^*$-equivalence. We recall that in \cite{Cinfidealwc22} we determined only the individual spectra $\lambda_V^T\lambda_V$ and $\lambda_U^T\lambda_U$ (which in the worst case was sufficient to ultimately obtain corresponding C-inf $\ell_1^*$ PT). While the calculations and supporting technicalities might on occasion be a bit heavy the overall methodology will be fairly similar to what we presented in \cite{Cinfidealwc22}. In fact, to make things easier to follow we will try to parallel the presentation from \cite{Cinfidealwc22} as much as possible. We start by recalling on $Q_1$ and introducing $Q_1^{\perp}$, and ${\cal Q}_1$
 \begin{eqnarray}
Q_1 & \triangleq & \lambda_V^T\lambda_V \nonumber \\
Q_1^{\perp}  &\triangleq & \lambda_U^T\lambda_U \nonumber \\
{\cal Q}_1 & \triangleq & Q_1Q_1^{\perp}.\label{eq:typacanl0a}
\end{eqnarray}
We also recall on the definitions of $Q$ and $Q^{\perp}$ and introduce ${\cal Q}$ in the following way
\begin{eqnarray}
Q & \triangleq & \lp (I^{(l)})^T\bV^{\perp}(\bV^{\perp})^T I^{(l)}\rp^{-1} -I \nonumber \\
Q^{\perp}  &\triangleq & \lp (I^{(l)})^T\bU^{\perp}(\bU^{\perp})^T I^{(l)}\rp^{-1} -I \nonumber \\
{\cal Q} & \triangleq & QQ^{\perp}.\label{eq:typacanl0}
\end{eqnarray}

\subsubsection{The spectrum of ${\cal Q}_1 \triangleq \lambda_V^T\lambda_V\lambda_U^T\lambda_U$ -- theoretical considerations}
\label{sec:fpteqvspecac}

Since $(Q_1,Q)$ and $(Q_1^{\perp},Q^{\perp})$ are statistically identical pairs, we will, for the time being, focus on only one of them, say $(Q_1,Q)$. To that end, we first recall the statistical relations within the pairs
\begin{eqnarray}
Q_1 & \triangleq  &\Lambda_V^T\Lambda_V \nonumber \\
Q  & \triangleq &  \lp (I^{(l)})^T\bV^{\perp}(\bV^{\perp})^T I^{(l)}\rp^{-1} -I=D^{-1}-I \nonumber \\
Sp(Q_1) & \Longleftrightarrow_{\setminus 0} & Sp(Q), \label{eq:typacanl1}
\end{eqnarray}
where $\Longleftrightarrow_{\setminus0}$ stands for the spectral equivalence outside the zeros eigenvalues. We will also find it convenient to work with the spectrum of $Q$. Later on we will make the necessary adjustments so that the results fully fit the spectrum of $Q_1$. To start things off we first note
\begin{eqnarray}
G_Q(z) = G_{D^{-1}}(z+1). \label{eq:typacanl2}
\end{eqnarray}
To see that (\ref{eq:typacanl2}) indeed holds, we first observe that the spectral functions of $Q$ and $D^{-1}$, $f_Q(x)$ and $f_{D^{-1}}(x)$, can be connected in the following way
\begin{eqnarray}
f_Q(x) = f_{D^{-1}}(x+1). \label{eq:typacanl3}
\end{eqnarray}
Then from (\ref{eq:typwcanl6}) we have
\begin{equation}
G_Q(z)=\int\frac{f_Q(x)}{z-x}dx = \int\frac{f_{D^{-1}}(x+1)}{z-x}dx
= \int\frac{f_{D^{-1}}(x+1)}{z+1-(x+1)}dx = \int\frac{f_{D^{-1}}(x)}{z+1-x}dx=
G_{D^{-1}}(z+1). \label{eq:typacanl4}
\end{equation}
From (\ref{eq:typacanl1}) and (\ref{eq:typacanl2}) we also have
\begin{eqnarray}
R_Q(z) = R_{D^{-1}}(z)-1. \label{eq:typacanl5}
\end{eqnarray}
Namely, (\ref{eq:typwcanl8}) first gives
\begin{eqnarray}
R_Q(G_Q(z)) =z-\frac{1}{G_Q(z)}, \label{eq:typacanl6}
\end{eqnarray}
and then
\begin{eqnarray}
\begin{array}{rclclcl}
R_Q(z) & = &G_Q^{-1}(z)-\frac{1}{z} & \Longleftrightarrow  & z & = & G_Q\lp R_Q(z) + \frac{1}{z}\rp \\
R_{D^{-1}}(z) & = & G_{D^{-1}}^{-1}(z)-\frac{1}{z} & \Longleftrightarrow  & z & = & G_{D^{-1}}\lp R_{D^{-1}}(z) + \frac{1}{z}\rp.
\end{array} \label{eq:typacanl7}
\end{eqnarray}
Combining (\ref{eq:typacanl2}) and (\ref{eq:typacanl7}) we obtain
\begin{eqnarray}
\begin{array}{lrclcl}
  &  G_Q\lp R_Q(z) + \frac{1}{z} \rp & = & G_{D^{-1}}\lp R_{D^{-1}}(z) + \frac{1}{z}\rp \\
\Longleftrightarrow \quad $ $ & G_Q\lp R_Q(z) + \frac{1}{z} \rp & = & G_{Q}\lp R_{D^{-1}}(z) + \frac{1}{z}-1\rp\\
\Longleftrightarrow \quad $ $ &  R_Q(z) + \frac{1}{z} & = &  R_{D^{-1}}(z) + \frac{1}{z}-1 \\
\Longleftrightarrow \quad $ $ &  R_Q(z) & = &  R_{D^{-1}}(z) -1,
\end{array}\label{eq:typacanl8}
\end{eqnarray}
which is exactly (\ref{eq:typacanl6}). From (\ref{eq:typwcanl9}) we further have
\begin{eqnarray}
S_Q(z) = \frac{1}{R_Q(zS_Q(z))} = \frac{1}{R_{D^{-1}}(zS_Q(z))-1}, \label{eq:typacanl9}
\end{eqnarray}
and
\begin{eqnarray}
 R_{D^{-1}}(zS_Q(z))=\frac{1}{S_Q(z)}+1. \label{eq:typacanl10}
\end{eqnarray}
Relying further on (\ref{eq:typwcanl9}) we also have
\begin{equation}
R_{D^{-1}}(zS_Q(z)) = \frac{1}{S_{D^{-1}}(zS_{Q}(z)R_{D^{-1}}(zS_Q(z))}
=\frac{1}{S_{D^{-1}}\lp zS_{Q}(z)\lp \frac{1}{S_Q(z)}+1 \rp \rp}
=\frac{1}{S_{D^{-1}}\lp z+ zS_{Q}(z)\rp}. \label{eq:typacanl11}
\end{equation}
A combination of (\ref{eq:typacanl10}) and (\ref{eq:typacanl11}) gives a way to connect the $S$-transforms of $D^{-1}$ and $Q$
\begin{eqnarray}
\frac{1}{S_{D^{-1}}\lp z+ zS_{Q}(z)\rp} = \frac{1}{S_Q(z)}+1. \label{eq:typacanl12}
\end{eqnarray}
From (\ref{eq:typacanl0}) and the key FPT principles (\ref{eq:typwcanl10}) we find
\begin{eqnarray}
S_{{\cal Q}}(z) = S_Q(z) S_{Q_{\perp}}(z) = (S_Q(z))^2, \label{eq:typacanl13}
\end{eqnarray}
where we used the fact that $Q$ and $Q^{\perp}$ are statistically identical and as such have the same $S$-transform. One can now rewrite (\ref{eq:typacanl12}) with $z\rightarrow zR_{{\cal Q}}(z)$ and utilize (\ref{eq:typwcanl9}) to obtain
\begin{eqnarray}
\begin{array}{lrcl}
 & \frac{1}{S_{D^{-1}}\lp zR_{{\cal Q}}(z)+ zR_{{\cal Q}}(z)S_{Q}(zR_{{\cal Q}}(z))\rp} & = &  \frac{1}{S_Q(zR_{{\cal Q}}(z))}+1 \\
\Longleftrightarrow \quad $ $ & \frac{1}{S_{D^{-1}}\lp zR_{{\cal Q}}(z)+ zR_{{\cal Q}}(z)\sqrt{S_{{\cal Q}}(zR_{{\cal Q}}(z))}\rp} & = &  \frac{1}{\sqrt{S_Q(zR_{{\cal Q}}(z))}}+1 \\
\Longleftrightarrow \quad $ $  & \frac{1}{S_{D^{-1}}\lp zR_{{\cal Q}}(z)+ z\sqrt{R_{{\cal Q}}(z)}\rp} & = &  \sqrt{R_{{\cal Q}}(z)}+1.
\end{array} \label{eq:typacanl14}
\end{eqnarray}
Replacing $z\rightarrow G_{{\cal Q}}(z)$, (\ref{eq:typacanl14}) can be further rewritten
\begin{eqnarray}
\begin{array}{lrcl}
  & \frac{1}{S_{D^{-1}}\lp zR_{{\cal Q}}(z)+ z\sqrt{R_{{\cal Q}}(z)}\rp} & = &  \sqrt{R_{{\cal Q}}(z)}+1\\
 \Longleftrightarrow \quad $ $  & \frac{1}{S_{D^{-1}}\lp G_{{\cal Q}}(z)R_{{\cal Q}}(G_{{\cal Q}}(z))+ G_{{\cal Q}}(z)\sqrt{R_{{\cal Q}}(G_{{\cal Q}}(z))}\rp} & = &  \sqrt{R_{{\cal Q}}(G_{{\cal Q}}(z))}+1.
\end{array} \label{eq:typacanl15}
\end{eqnarray}
From (\ref{eq:typwcanl8}) we find
\begin{eqnarray}
\begin{array}{lrcl}
  &  R_{\cal Q}(G_{\cal Q}(z)) + \frac{1}{G_{{\cal Q}}(z)} & = & z\\
 \Longleftrightarrow \quad $ $  & G_{{\cal Q}}(z) R_{\cal Q}(G_{\cal Q}(z)) & = & zG_{{\cal Q}}(z)-1.
\end{array} \label{eq:typacanl16}
\end{eqnarray}
Combining further (\ref{eq:typacanl15}) and (\ref{eq:typacanl16}) we also have
\begin{eqnarray}
\begin{array}{lrcl}
  & \frac{1}{S_{D^{-1}}\lp G_{{\cal Q}}(z)R_{{\cal Q}}(G_{{\cal Q}}(z))+ G_{{\cal Q}}(z)\sqrt{R_{{\cal Q}}(G_{{\cal Q}}(z))}\rp} & = &  \sqrt{R_{{\cal Q}}(G_{{\cal Q}}(z))}+1\\
  \Longleftrightarrow \quad $ $  & \frac{1}{S_{D^{-1}}\lp zG_{{\cal Q}}(z)-1+ \sqrt{G_{{\cal Q}}(z)}\sqrt{zG_{{\cal Q}}(z)-1}\rp} & = &  \sqrt{\frac{zG_{{\cal Q}}(z)-1}{G_{{\cal Q}}(z)}}+1.
\end{array} \label{eq:typacanl17}
\end{eqnarray}
As in \cite{Haag97} one has for the connection between the $S$-transforms of the matrix and its inverse
\begin{eqnarray}
S_D(z) = \frac{1}{S_{D^{-1}}(-1-z)}. \label{eq:typacanl18}
\end{eqnarray}
Keeping (\ref{eq:typacanl18}) in mind, one can rewrite (\ref{eq:typacanl17})  in the following way
\begin{eqnarray}
\begin{array}{lrcl}
  & \frac{1}{S_{D^{-1}}\lp G_{{\cal Q}}(z)R_{{\cal Q}}(G_{{\cal Q}}(z))+ G_{{\cal Q}}(z)\sqrt{R_{{\cal Q}}(G_{{\cal Q}}(z))}\rp} & = &  \sqrt{R_{{\cal Q}}(G_{{\cal Q}}(z))}+1\\
  \Longleftrightarrow \quad $ $  & S_{D}\lp -zG_{{\cal Q}}(z)- \sqrt{G_{{\cal Q}}(z)}\sqrt{zG_{{\cal Q}}(z)-1}\rp & = &  \sqrt{\frac{zG_{{\cal Q}}(z)-1}{G_{{\cal Q}}(z)}}+1.
\end{array} \label{eq:typacanl19}
\end{eqnarray}
We will make a $S_D(z)-G_D(z)$ connection below. however, before doing so, we will need to make certain adjustments in the $G_D(z)$ transform itself.

\vspace{.15in}
\noindent \underline{\bl{\textbf{i) Adjusting $G_D(z)$ for the difference between $\tilde{D}$ and $\bar{D}$}}}

\noindent  We now briefly recall on the connection between $\tilde{D}$, $\bar{D}$, and $D$. First, from  (\ref{eq:typwcanl18}) and (\ref{eq:typwcanl21}) we have
\begin{eqnarray}
\tilde{D} &  = &  \bV^{\perp}(\bV^{\perp})^T\bU_D^{\perp}(\bU_D^{\perp})^T \nonumber \\
\bar{D} &  = &  (\bU_D^{\perp})^T\bV^{\perp}(\bV^{\perp})^T\bU_D^{\perp}. \label{eq:typacanl20}
\end{eqnarray}
The spectra of $\tilde{D}$ and $\bar{D}$ are modulo scalings practically identical. Since $\bar{D}$ has all the eigenvalues that $\tilde{D}$ has with $l=\eta n$ zero eigenvalues removed one can connect their $G$-transforms in the following way.
\begin{eqnarray}
G_{\bar{D}}(z) &  = & \frac{1}{1-\eta}  \lp G_{\tilde{D}}(z) - \frac{\eta}{z} \rp. \label{eq:typacanl21}
\end{eqnarray}
To see that (\ref{eq:typacanl21}) is indeed true we first connect the spectral pdfs of $\tilde{D}$ and $\bar{D}$
\begin{eqnarray}
f_{\bar{D}}(x) &  = & \frac{1}{1-\eta}  \lp f_{\tilde{D}}(x) - \eta\delta(x) \rp. \label{eq:typacanl22}
\end{eqnarray}
Then from (\ref{eq:typwcanl6}) we have
\begin{eqnarray}
G_{\bar{D}}(x)=\int \frac{f_{\bar{D}}(x)}{z-x}dx  = \frac{1}{1-\eta}  \lp \int \frac{f_{\tilde{D}}(x)}{z-x}dx - \eta\int\frac{\delta(x)}{z-x}dx \rp
= \frac{1}{1-\eta}  \lp G_{\tilde{D}}(z) - \frac{\eta}{z} \rp. \label{eq:typacanl23}
\end{eqnarray}
Connecting beginning and end in (\ref{eq:typacanl23}) we obtain (\ref{eq:typacanl21}).

\vspace{.15in}
\noindent \underline{\bl{\textbf{ii) Adjusting $G_D(z)$ for the difference between $Q_1$ and $Q$}}}

\noindent  We recall that $Q_1$ has the same eigenvalues as $Q$ minus $n-l-k$ zero eigenvalues (when $n-l-k\leq 0$ that means that $Q_1$ has all the eigenvalues of $Q$ plus $|n-l-k|$ zero eigenvalues). To account for this difference we find it useful to introduce a matrix $D_1$ obtained by removing/adding $|n-(l+k)|$ ones into the spectrum of $D$. As these added ones are inversion invariant they remain in the spectrum after the inversion. This means that after the inversion of $D_1$ and subtraction of the identity matrix they become zeros and basically have an effect on $Q$ as if $|n-(l+k)|$ zeros were added or removed which is exactly what we need to account for the difference between $Q_1$ and $Q$. To put everything in the right mathematical context, let $D_1$ be a matrix with the Haar distributed eigen-space basis and the spectral function defined int he following way
\begin{eqnarray}
f_{D_1}=\frac{1}{1-\eta-(1-(\beta+\eta))}\lp f_{\tilde{D}}-\eta\delta(x)-(1-(\beta+\eta))\delta(x-1)\rp, \label{eq:typacanl23a}
\end{eqnarray}
where we have now taken into the account the above mentioned adjusting between $\tilde{D}$ and $\bar{D}$ ($\bar{D}$ and $D$ have identical spectral functions). Utilizing again (\ref{eq:typwcanl6}) we similarly to (\ref{eq:typacanl23}) have
\begin{eqnarray}
G_{D_1}(z) = \frac{1}{1-\eta-(1-\beta+\eta)}  \lp G_{\tilde{D}}(z) - \frac{\eta}{z}  - \frac{1-(\beta+\eta)}{z-1}\rp. \label{eq:typacanl24}
\end{eqnarray}
Recalling once again on (\ref{eq:typwcanl8}) we have
\begin{eqnarray}
\begin{array}{rrcl}
 & R_{D_1}(G_{D_1}(z))+\frac{1}{G_{D_1}(z)} & = & z \\
\Longleftrightarrow  \quad $ $ & R_{D_1}(z)+\frac{1}{z} & = & G_{D_1}^{-1}(z).
\end{array}\label{eq:typacanl25}
\end{eqnarray}
After taking $z\rightarrow R_{D_1}(z)+\frac{1}{z}$ we can rewrite (\ref{eq:typacanl24}) as
\begin{eqnarray}
G_{D_1}\lp R_{D_1}(z)+\frac{1}{z}\rp
 = \frac{1}{\beta}  \lp G_{\tilde{D}}\lp R_{D_1}(z)+\frac{1}{z}\rp - \frac{\eta}{R_{D_1}(z)+\frac{1}{z}}  - \frac{1-(\beta+\eta)}{R_{D_1}(z)+\frac{1}{z}-1}\rp, \label{eq:typacanl26}
\end{eqnarray}
 and after utilizing (\ref{eq:typacanl25})
\begin{eqnarray}
z = \frac{1}{\beta}  \lp G_{\tilde{D}}\lp R_{D_1}(z)+\frac{1}{z}\rp - \frac{\eta}{R_{D_1}(z)+\frac{1}{z}}  - \frac{1-(\beta+\eta)}{R_{D_1}(z)+\frac{1}{z}-1}\rp. \label{eq:typacanl27}
\end{eqnarray}
After another replacement, $z\rightarrow zS_{D_1}(z)$, (\ref{eq:typacanl27}) becomes
{\small\begin{equation}
 zS_{D_1}(z)= \frac{1}{\beta}  \lp G_{\tilde{D}}\lp R_{D_1}(zS_{D_1}(z))+\frac{1}{zS_{D_1}(z)}\rp - \frac{\eta}{R_{D_1}(zS_{D_1}(z))+\frac{1}{zS_{D_1}(z)}}  - \frac{1-(\beta+\eta)}{R_{D_1}(zS_{D_1}(z))+\frac{1}{zS_{D_1}(z)}-1}\rp. \label{eq:typacanl28}
\end{equation}}
Using (\ref{eq:typwcanl9}) we from (\ref{eq:typacanl28}) further find
\begin{equation}
 zS_{D_1}(z)= \frac{1}{\beta}  \lp G_{\tilde{D}}\lp \frac{1}{S_{D_1}(z)}+\frac{1}{zS_{D_1}(z)}\rp - \frac{\eta}{\frac{1}{S_{D_1}(z)}+\frac{1}{zS_{D_1}(z)}}  - \frac{1-(\beta+\eta)}{\frac{1}{S_{D_1}(z)}+\frac{1}{zS_{D_1}(z)}-1}\rp, \label{eq:typacanl29}
\end{equation}
and
\begin{equation}
 zS_{D_1}(z)= \frac{1}{\beta}  \lp G_{\tilde{D}}\lp \frac{z+1}{zS_{D_1}(z)}\rp - \frac{\eta}{\frac{z+1}{zS_{D_1}(z)}}  - \frac{1-(\beta+\eta)}{\frac{z+1}{zS_{D_1}(z)}-1}\rp. \label{eq:typacanl30}
\end{equation}
Taking $D\rightarrow D_1$ in (\ref{eq:typacanl1}) and correspondingly denoting $Q\rightarrow Q_1$ and ${\cal Q}\rightarrow {\cal Q}_1$, one can repeat all the steps between (\ref{eq:typacanl1}) and (\ref{eq:typacanl19}) to arrive at the following
\begin{eqnarray}
 S_{D_1}\lp -zG_{{\cal Q}_1}(z)- \sqrt{G_{{\cal Q}_1}(z)}\sqrt{zG_{{\cal Q}_1}(z)-1}\rp & = &  \sqrt{\frac{zG_{{\cal Q}_1}(z)-1}{G_{{\cal Q}_1}(z)}}+1. \label{eq:typacanl31}
\end{eqnarray}
Setting
\begin{eqnarray}
  z_1(z) & \triangleq & -zG_{{\cal Q}_1}(z)- \sqrt{G_{{\cal Q}_1}(z)}\sqrt{zG_{{\cal Q}_1}(z)-1} \nonumber \\
  y(z) & \triangleq & \frac{z_1(z)+1}{z_1(z)S_{D_1}(z_1(z))}, \label{eq:typacanl32}
\end{eqnarray}
one has from (\ref{eq:typacanl31})
\begin{eqnarray}
 S_{D_1}(z_1(z)) & = &  \sqrt{\frac{zG_{{\cal Q}_1}(z)-1}{G_{{\cal Q}_1}(z)}}+1. \label{eq:typacanl33}
\end{eqnarray}
After taking $z\rightarrow z_1$ and rewriting (\ref{eq:typacanl30}) one finally obtains
\begin{equation}
 \frac{z_1(z)+1}{y(z)}= \frac{1}{\beta}  \lp G_{\tilde{D}}(y(z)) - \frac{\eta}{y(z)}  - \frac{1-(\beta+\eta)}{y(z)-1}\rp. \label{eq:typacanl34}
\end{equation}

We summarize the above results in the following lemma.

\begin{lemma}
  Let ${\cal Q}_1$ be as in (\ref{eq:typacanl0a}). Then its $G$-transform, $G_{{\cal Q}_1}(z)$, satisfies
\begin{equation}
 \frac{z_1(z)+1}{y(z)}= \frac{1}{\beta}  \lp G_{\tilde{D}}(y(z)) - \frac{\eta}{y(z)}  - \frac{1-(\beta+\eta)}{y(z)-1}\rp. \label{eq:typaclemma1eq1}
\end{equation}
  with $z_1(z)$ and $y(z)$ as in (\ref{eq:typacanl32}),  $S_{D_1}(z_1(z))$ as in (\ref{eq:typacanl33}), and $G_{\bar{D}}(y(z))$ as in Lemma \ref{lemma:typwclemma1a}.
  \label{lemma:typaclemma1}
\end{lemma}

A combination of Lemma \ref{lemma:typwclemma1a} (where $G_{\tilde{D}}(\cdot)$ is explicitly given) and (\ref{eq:typacanl32})-(\ref{eq:typacanl34}) is then sufficient to determine $G_{{\cal Q}_1}(z)$ . Utilizing (\ref{eq:typwcanl7}) then enables one to fully determine the spectral distribution. This is a generic procedure that in principle works. Below we will move things a step further and provide a more detailed analysis of the edges of the spectrum as they play a critical role in the $\ell_0^*-\ell_1^*$-equivalence. It will turn out that one can provide their a sufficiently explicit characterization so that the explicit closed form for the corresponding C-inf phase transitions can again be obtained. Later on we will return to the above described procedure for determining the entire spectrum of ${\cal Q}_1$ and show what type of results such a procedure actually produces.

\subsubsection{Explicit characterization of ${\cal Q}_1$'s spectral edges}
\label{sec:fpteqvspecac}

As we have seen earlier, the upper edge of the spectrum of ${\cal Q}$ (or ${\cal Q}_1$), $\lambda_{max}({\cal Q})=\lambda_{max}({\cal Q}_1)$ is directly related to the success of the $\ell_1^*$-minimization heuristic in causal inference. More precisely, as Corollary \ref{cor:cinfcor3} states, one will have the $\ell_0^*-\ell_1^*$-equivalence if and only if $\lambda_{max}({\cal Q})=\lambda_{max}({\cal Q}_1)\leq 1$. Clearly, an explicit characterization of $\lambda_{max}({\cal Q})=\lambda_{max}({\cal Q}_1)$ will be sufficient to explicitly characterize the $\ell_0^*-\ell_1^*$-equivalence. That will then be enough to conclude when $\ell_1^*$ can be used reliable to handle the casual inference.

To provide an explicit characterization of $\lambda_{max}({\cal Q})=\lambda_{max}({\cal Q}_1)\leq 1$ we rely on the results that we presented in the previous section. We start by observing that the spectral function of ${\cal Q}_1$, $f_{{\cal Q}_1}(x)$, can be obtained by utilizing (\ref{eq:typwcanl7}) and the above discussed $G_{{\cal Q}_1}(z)$ transform. Moreover, at the edge of the spectrum $G_{{\cal Q}_1}(z)$ should be real (the edge of the spectrum is actually the breaking point where the $G_{{\cal Q}_1}(z)$ becomes complex, i.e. starts having a nonzero imaginary part). That basically means that at the edge of the spectrum one should have (\ref{eq:typaclemma1eq1}) satisfied for a real $G_{{\cal Q}_1}(z)$. Moreover, since our targeted edge of the spectrum is one that means that (\ref{eq:typaclemma1eq1}) needs to be satisfied for a real $G_{{\cal Q}_1}(1)$. Rewriting (\ref{eq:typacanl32})-(\ref{eq:typacanl34}) for $z=1$ gives
\begin{eqnarray}
  z_1(1) & = & -G_{{\cal Q}_1}(1)- \sqrt{G_{{\cal Q}_1}(1)}\sqrt{G_{{\cal Q}_1}(1)-1} \nonumber \\
  y(1) & \triangleq & \frac{z_1(1)+1}{z_1(1)S_{D_1}(z_1(1))}, \label{eq:typacanl35}
\end{eqnarray}
and
\begin{eqnarray}
 S_{D_1}(z_1(1)) & = &  \sqrt{\frac{G_{{\cal Q}_1}(1)-1}{G_{{\cal Q}_1}(1)}}+1 = -\frac{z_1(1)}{G_{{\cal Q}_1}(1)}, \label{eq:typacanl36}
\end{eqnarray}
and
 \begin{equation}
 \frac{z_1(1)+1}{y(1)}= \frac{1}{\beta}  \lp G_{\tilde{D}}(y(1)) - \frac{\eta}{y(1)}  - \frac{1-(\beta+\eta)}{y(1)-1}\rp. \label{eq:typacanl37}
\end{equation}
From (\ref{eq:typacanl35}) one further finds
\begin{eqnarray}
 G_{{\cal Q}_1}(1) & = & -\frac{(z_1(1))^2}{1+2z_1(1)} \nonumber \\
  y(1) & = & -\frac{(z_1(1)+1)G_{{\cal Q}_1}(1)}{(z_1(1))^2}=\frac{z_1(1)+1}{1+2z_1(1)}. \label{eq:typacanl38}
\end{eqnarray}
The second equality then also gives
\begin{eqnarray}
   z_1(1) & = & \frac{y(1)-1}{1-2y(1)}, \label{eq:typacanl39}
\end{eqnarray}
and
\begin{eqnarray}
   z_1(1)+1 & = & \frac{y(1)}{2y(1)-1}. \label{eq:typacanl40}
\end{eqnarray}
Plugging (\ref{eq:typacanl40}) into (\ref{eq:typacanl37}) one has
 \begin{equation}
 \frac{1}{2y(1)-1}= \frac{1}{\beta}  \lp G_{\tilde{D}}(y(1)) - \frac{\eta}{y(1)}  - \frac{1-(\beta+\eta)}{y(1)-1}\rp, \label{eq:typacanl41}
\end{equation}
or
 \begin{equation}
 \zeta_1(y) \triangleq -\frac{1}{2y-1} + \frac{1}{\beta}  \lp G_{\tilde{D}}(y) - \frac{\eta}{y}  - \frac{1-(\beta+\eta)}{y-1}\rp = 0. \label{eq:typacanl41}
\end{equation}
Utilizing $G_{\tilde{D}}(z)$ (with the ``$-$" sign as the lower edge in the bulk of the spectrum of $\tilde{D}$ corresponds to the upper edge in the spectrum of ${\cal Q}$) from Lemma \ref{lemma:typwclemma1a} we further have
 \begin{equation}
 \zeta_1(y) = -\frac{1}{2y-1}
 + \frac{2\beta-1}{2\beta(y-1)}
  +\frac{1}{2\beta y(y-1)}
 \lp -\beta+\eta-\sqrt{(y-(\beta+\eta))^2+4\beta\eta(y-1)} \rp = 0, \label{eq:typacanl42}
\end{equation}
 and
 \begin{eqnarray}
 \zeta_1(y) & = & \frac{-2\beta y(y-1)+y(2\beta-1)(2y-1) + (2y-1) \lp -\beta+\eta-\sqrt{(y-(\beta+\eta))^2+4\beta\eta(y-1)} \rp}{2\beta y(y-1)(2y-1)} \nonumber \\
  & = & \frac{2(\beta-1)y^2+ (1-2\beta+2\eta)y  +\beta-\eta - (2y-1)\sqrt{(y-(\beta+\eta))^2+4\beta\eta(y-1)} }{2\beta y(y-1)(2y-1)}. \label{eq:typacanl43}
\end{eqnarray}
Setting
 \begin{eqnarray}
 \zeta_2(y) & \triangleq & 2(\beta-1)y^2+ (1-2\beta+2\eta)y  +\beta-\eta - (2y-1)\sqrt{(y-(\beta+\eta))^2+4\beta\eta(y-1)} \nonumber \\
  \zeta(y) & \triangleq & (2(\beta-1)y^2+ (1-2\beta+2\eta)y  +\beta-\eta)^2 - ((2y-1)\sqrt{(y-(\beta+\eta))^2+4\beta\eta(y-1)})^2, \label{eq:typacanl44}
\end{eqnarray}
we easily have
 \begin{eqnarray}
 \zeta_1(y) =0 \quad \Longleftrightarrow \quad \zeta_2(y)=0 \quad \Longleftrightarrow \quad \zeta(y)=0.  \label{eq:typacanl45}
\end{eqnarray}
We therefore below focus on $\zeta(y)$. After squaring and grouping the terms we have
 \begin{eqnarray}
   \zeta(y) =4\beta( c_3y^4+c_2y^3+c_1y^2+c_0y+c_{00}), \label{eq:typacanl46}
\end{eqnarray}
with
 \begin{eqnarray}
c_3 & =& \beta-2 \nonumber \\
c_2 & =& 5-2\beta-2\eta \nonumber \\
c_1 & =& \beta-4+3\eta \nonumber \\
c_0 & =& 1-\eta \nonumber \\
c_{00} & =& 0. \label{eq:typacanl47}
\end{eqnarray}
From (\ref{eq:typacanl46}) we then also have
\begin{eqnarray}
   \zeta(y) =4\beta y( c_3y^3+c_2y^2+c_1y+c_0). \label{eq:typacanl48}
\end{eqnarray}
Since we are interested in an edge or a breaking point of the spectrum $\zeta(y)$ should touch zero for certain $y$ which means that it should have a stationary point at such $y$. To find such a stationary point we take the derivative
\begin{eqnarray}
   \frac{d\lp \frac{\zeta(y)}{4\beta y}\rp}{dy} = 3c_3y^2+2c_2y+c_1=0. \label{eq:typacanl49}
\end{eqnarray}
Solving over $y$ gives
\begin{eqnarray}
y=\frac{-c_2+\sqrt{c_2^2-3c_1c_3}}{3c_3}. \label{eq:typacanl50}
\end{eqnarray}
Setting
\begin{eqnarray}
r\triangleq c_2^2-3c_1c_3=1+\beta^2+4\eta^2-2\beta-2\eta-\beta\eta, \label{eq:typacanl51}
\end{eqnarray}
we have from (\ref{eq:typacanl50})
\begin{eqnarray}
y_{opt}=\frac{-c_2+\sqrt{r}}{3c_3}. \label{eq:typacanl52}
\end{eqnarray}
First we set
\begin{eqnarray}
   \zeta_3(y) \triangleq c_3 y^3+c_2 y^2+c_1y+c_0. \label{eq:typacanl53}
\end{eqnarray}
Clearly, from (\ref{eq:typacanl48}) one has
\begin{eqnarray}
   \zeta(y) =4\beta y \zeta_3(y). \label{eq:typacanl53a}
\end{eqnarray}
Then we plug the value for $y_{opt}$ from (\ref{eq:typacanl52}) and after a bit of algebraic transformations obtain
\begin{eqnarray}
   \zeta_3(y_{opt}) = -2(\sqrt{r})^3-c_2^3+3r c_2+27c_3^2c_0. \label{eq:typacanl54}
\end{eqnarray}
From (\ref{eq:typacanl47}) we first have
 \begin{eqnarray}
 c_2 & =& -2c_3-1+2c_0\nonumber \\
c_1 & =& c_3-3c_0+1, \label{eq:typacanl55}
\end{eqnarray}
and then from (\ref{eq:typacanl51})
\begin{eqnarray}
r = c_3^2+1+4c_0^2+c_3-4c_0+c_0c_3. \label{eq:typacanl56}
\end{eqnarray}
Combining (\ref{eq:typacanl54})-(\ref{eq:typacanl56}) after a bit of additional algebraic transformations gives
\begin{eqnarray}
   \zeta_3(y_{opt}) =
   -2(\sqrt{r})^3+2c_3^3-3 c_3+6c_3c_0^2+3c_3^2+3c_3c_0+3c_0c_3^2-2-24 c_0^2+12 c_0+16 c_0^3. \label{eq:typacanl57}
\end{eqnarray}
Below we show that
\begin{eqnarray}
 c_3^2=-1-2c_3+4c_0-4c_0^2 \quad \Longleftrightarrow \quad  \zeta_3(y_{opt}) =0. \label{eq:typacanl58}
\end{eqnarray}
We first use (\ref{eq:typacanl58}) to systematically linearize $\zeta_3(y_{opt})$ in $c_3$ and obtain
\begin{eqnarray}
   \zeta_3(y_{opt}) =
-2(\sqrt{r})^3+(-3c_3+5c_3c_0-2c_3c_0^2-1-8c_0^2+5c_0+4c_0^3). \label{eq:typacanl59}
\end{eqnarray}
Transforming further we also have
\begin{eqnarray}
   \zeta_3(y_{opt}) & = & -2(\sqrt{r})^3+(-3c_3+5c_3c_0-2c_3c_0^2-1-8c_0^2+5c_0+4c_0^3)\nonumber \\
   & = & -2(\sqrt{r})^3+(c_3(c_0-1)(3-2c_0)+(-4c_0+1+4c_0^2)(c_0-1)) \nonumber \\
   & = & 2(\sqrt{c_3}\sqrt{c_0-1})^3+(c_3(c_0-1)(3-2c_0)+(-4c_0+1+4c_0^2)(c_0-1)) \nonumber \\
   & = & \lp 2(\sqrt{c_3})^3\sqrt{c_0-1}+(c_3(3-2c_0)+(-4c_0+1+4c_0^2))\rp(c_0-1). \label{eq:typacanl60}
\end{eqnarray}
where the third equality follows after noting that with condition (\ref{eq:typacanl58}) in place $r$ in \ref{eq:typacanl56}) becomes
\begin{eqnarray}
 c_3^2=-1-2c_3+4c_0-4c_0^2 \quad \Longrightarrow \quad r =-c_3+c_0c_3. \label{eq:typacanl61}
\end{eqnarray}
We find it useful to rewrite (\ref{eq:typacanl60}) as
\begin{eqnarray}
   \zeta_3(y_{opt}) & = &   \zeta_3^{(1)}(y_{opt})+   \zeta_3^{(2)}(y_{opt}), \label{eq:typacanl61a}
\end{eqnarray}
where
\begin{eqnarray}
   \zeta_3^{(1)}(y_{opt}) & \triangleq &   2(\sqrt{c_3})^3\sqrt{c_0-1} \nonumber \\
   \zeta_3^{(2)}(y_{opt}) &\triangleq &    c_3(3-2c_0)+(-4c_0+1+4c_0^2). \label{eq:typacanl62}
   \end{eqnarray}
We then look at the squared values of these quantities. First we start with $\zeta_3^{(1)}(y_{opt})$
\begin{eqnarray}
   (\zeta_3^{(1)}(y_{opt}))^2  =   4c_3^3(c_0-1), \label{eq:typacanl63}
   \end{eqnarray}
and utilize the condition (\ref{eq:typacanl58}) to systematically linearize in $c_3$. First we remove the cubic $c_3$ term to arrive at the following
 \begin{eqnarray}
   (\zeta_3^{(1)}(y_{opt}))^2 &  = & 4c_3^3(c_0-1) \nonumber \\
      &  = & 4c_3(-1-2c_3+4c_0-4c_0^2)(c_0-1) \nonumber \\
      &  = &  -8c_3^2c_0+32c_3c_0^2-16c_3c_0^3-8-12c_3+32c_0-32c_0^2-20c_3c_0 , \label{eq:typacanl64}
   \end{eqnarray}
and apply the same procedure again to arrive at a fully linearized form
 \begin{eqnarray}
   (\zeta_3^{(1)}(y_{opt}))^2  = 4(c_3((-2c_0^2+c_0+1)(2c_0-3))-2(1-c_0)(2c_0-1)^2). \label{eq:typacanl65}
   \end{eqnarray}
Then we turn to $\zeta_3^{(2)}(y_{opt})$
\begin{eqnarray}
   (\zeta_3^{(2)}(y_{opt}))^2  =  (c_3(3-2c_0)+(-4c_0+1+4c_0^2))^2, \label{eq:typacanl66}
   \end{eqnarray}
and again utilize the condition (\ref{eq:typacanl58}) to  linearize in $c_3$. This time the procedure is simpler as there is only a quadratic term in $c_3$ and there is no need to apply the procedure from above in two steps. Instead only one step suffices and we have
 \begin{eqnarray}
   (\zeta_3^{(2)}(y_{opt}))^2 &  = & (c_3(3-2c_0)+(-4c_0+1+4c_0^2))^2 \nonumber \\
      &  = & c_3^2(3-2c_0)^2+(-4c_0+1+4c_0^2)^2+2(-4c_0+1+4c_0^2)c_3(3-2c_0) \nonumber \\
      &  = &  (-1-2c_3+4c_0-4c_0^2)(3-2c_0)^2+(-4c_0+1+4c_0^2)^2+2(-4c_0+1+4c_0^2)c_3(3-2c_0) \nonumber \\
      &  = & -2c_3(9-12c_0+4c_0^2-(-4c_0+1+4c_0^2)(3-2c_0))-(-4c_0+1+4c_0^2)(8-8c_0)  \nonumber \\
      & =  & 4(c_3((-2c_0^2+c_0+1)(2c_0-3))-2(1-c_0)(2c_0-1)^2). \label{eq:typacanl67}
   \end{eqnarray}
Comparing (\ref{eq:typacanl65}) and (\ref{eq:typacanl67}) we have
 \begin{eqnarray}
   (\zeta_3^{(1)}(y_{opt}))^2 &  = &    (\zeta_3^{(2)}(y_{opt}))^2. \label{eq:typacanl68}
   \end{eqnarray}
Now we will show that one also has $(\zeta_3^{(1)}(y_{opt}))^2=-  (\zeta_3^{(2)}(y_{opt}))^2$. We again look at the condition in (\ref{eq:typacanl58}) and replace the values for $c_0$ and $c_3$ from (\ref{eq:typacanl48}) to obtain
\begin{eqnarray}
\begin{array}{rrcl}
 &  c_3^2 & = & -1-2c_3+4c_0-4c_0^2 \\
 \Longleftrightarrow \quad $ $ & (\beta-2)^2 & = & -1 - 2(\beta-2)+4(1-\eta)-4(1-\eta)^2 \\
  \Longleftrightarrow \quad $ $ & (\beta-2)^2  + 2(\beta-2) +1& = & 4(1-\eta)-4(1-\eta)^2 \\
  \Longleftrightarrow \quad $ $ & (\beta-2+1)^2 & = & 4\eta(1-\eta) \\
   \Longleftrightarrow \quad $ $ & \beta & = & 1-2\sqrt{\eta(1-\eta)}.
 \end{array} \label{eq:typacanl69}
\end{eqnarray}
From (\ref{eq:typacanl62}) we then also have
\begin{eqnarray}
   \zeta_3^{(1)}(y_{opt}) & \triangleq &   2(\sqrt{c_3})^3\sqrt{c_0-1}
    = 2(\sqrt{\beta-2})^3\sqrt{-\eta} = 2(2-\beta)\sqrt{\eta(2-\beta)}\geq 0. \label{eq:typacanl70}
   \end{eqnarray}
Similarly, we have
\begin{eqnarray}
    \zeta_3^{(2)}(y_{opt}) &\triangleq &   c_3(3-2c_0)+(-4c_0+1+4c_0^2) \nonumber \\
     & = &    (\beta-2)(1+2\eta)+(1-2\eta)^2  \nonumber \\
     & = &    (-1-2\sqrt{\eta(1-\eta)})(1+2\eta)+(1-2\eta)^2  \nonumber \\
     & = &    -2\sqrt{\eta(1-\eta)}(1+2\eta)-6\eta+4\eta^2  \nonumber \\
     & \leq &    -2\sqrt{\eta(1-\eta)}(1+2\eta)-6\eta+4\eta  \nonumber \\
     & = &    -2\sqrt{\eta(1-\eta)}(1+2\eta)-2\eta \nonumber \\
     & \leq &   0. \label{eq:typacanl71}
   \end{eqnarray}
A combination of (\ref{eq:typacanl61a}), (\ref{eq:typacanl62}), (\ref{eq:typacanl68}), (\ref{eq:typacanl70}), and (\ref{eq:typacanl71}) finally gives
\begin{eqnarray}
\begin{array}{rrcl}
 &  (\zeta_3^{(1)}(y_{opt}))^2 & \overset{{\tiny(\ref{eq:typacanl68})}}{=} & (\zeta_3^{(2)}(y_{opt}))^2 \\
 \overset{{\tiny(\ref{eq:typacanl70}),(\ref{eq:typacanl71})}}{\Longleftrightarrow} \quad $ $ & \zeta_3^{(1)}(y_{opt})  & = & -\zeta_3^{(2)}(y_{opt}) \\
  \Longleftrightarrow \hspace{.26in} $ $ & \zeta_3^{(1)}(y_{opt})+\zeta_3^{(2)}(y_{opt}) & = & 0 \\
 \overset{{\tiny(\ref{eq:typacanl61a})}}{\Longleftrightarrow} \hspace{.13in}  \quad $ $ &  \zeta_3(y_{opt}) & = & 0.
 \end{array} \label{eq:typacanl72}
\end{eqnarray}
Moreover, a combination of (\ref{eq:typacanl58}), (\ref{eq:typacanl69}), and (\ref{eq:typacanl72}) gives
\begin{eqnarray}
\beta  =  1-2\sqrt{\eta(1-\eta)}\quad \Longleftrightarrow \quad  c_3^2=-1-2c_3+4c_0-4c_0^2 \quad \Longleftrightarrow \quad  \zeta_3(y_{opt}) =0. \label{eq:typacanl73}
\end{eqnarray}
After combining (\ref{eq:typacanl53a}) and (\ref{eq:typacanl73}) one then also has
\begin{eqnarray}
\beta  =  1-2\sqrt{\eta(1-\eta)}\quad \Longleftrightarrow \quad  c_3^2=-1-2c_3+4c_0-4c_0^2 \quad \Longleftrightarrow \quad  \zeta_3(y_{opt}) =0
\quad \Longleftrightarrow \quad  \zeta(y_{opt}) =0. \label{eq:typacanl74}
\end{eqnarray}
From (\ref{eq:typacanl45}) one then has that for $y_{opt}$
\begin{eqnarray}
   \zeta_1(y_{opt})=\zeta_2(y_{opt}) =0, \label{eq:typacanl75}
\end{eqnarray}
which means that $y=y_{opt}$ is indeed a choice for $y$ that ensures that functional equation used to determine $G_{{\cal Q}_1}(z)$ is satisfied. Moreover, since the derivative condition is met as well, i.e. since  $\zeta(y_{opt})=0$, one has that not only is $y_{opt}$ a point where $\zeta(y)$ crosses zero, it is actually a point where it touches zero. That is exactly what is needed to determine an edge of the spectrum. Since we operated using the ``$-$" sign in the definition of $G_{\tilde{D}}(z)$  that means (based on the considerations from \cite{Cinfidealwc22}) that we have determined the lower edge in the corresponding spectrum of $\tilde{D}$ (or any of $\bar{D}$ and $D$) which after the inversion means that we have determined the upper edge in the spectrum of ${\cal Q}_1$ or ${\cal Q}$.

One can even explicitly determine $y_{opt}$. From (\ref{eq:typacanl47}), (\ref{eq:typacanl52}), (\ref{eq:typacanl55}), and (\ref{eq:typacanl61}) we obtain
\begin{eqnarray}
y_{opt}=\frac{-c_2+\sqrt{r}}{3c_3}=\frac{-(5-2\beta-2\eta)+\sqrt{\eta(2-\beta)}}{3(\beta-2)}. \label{eq:typacanl76}
\end{eqnarray}

In Figure \ref{fig:yoptofeta} we show $y_{opt}$ as a function of $\eta$. The whole mechanism of ``touching zero" as $\beta$ decreases is shown in Figure \label{fig:zeta1ofy} for $\eta=0.9$. As can be seen from the figure, for $\beta> 1-2\sqrt{\eta(1-\eta)}=0.4$ $\zeta_1(y)$ remains below zero one therefore can not be a part of the spectrum. On the other hand, for $\beta\leq 1-2\sqrt{\eta(1-\eta)}=0.4$ $\zeta_1(y)$ does intersect zero line which implies  that one is now in the spectrum (there is $y=y(11)$ and consequently a real $G_{{\cal Q}_1}(1)$ such that $\zeta_1(y)=0$). The borderline or the breaking point happens exactly when the $\zeta_1(y)$ curve touches the zero line. As figure indicates that happens for $y=y_{opt}=0.25$, exactly as the theory predicts.

\begin{figure}[htb]
\centering
\centerline{\epsfig{figure=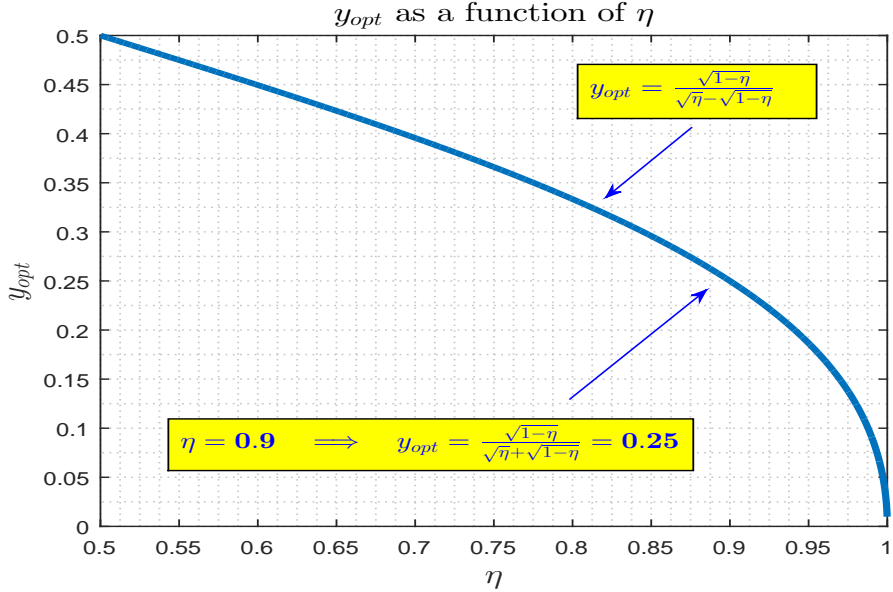,width=13.5cm,height=8cm}}
\caption{$y_{opt}$ as a function of $\eta$}
\label{fig:yoptofeta}
\end{figure}

\begin{figure}[htb]
\centering
\centerline{\epsfig{figure=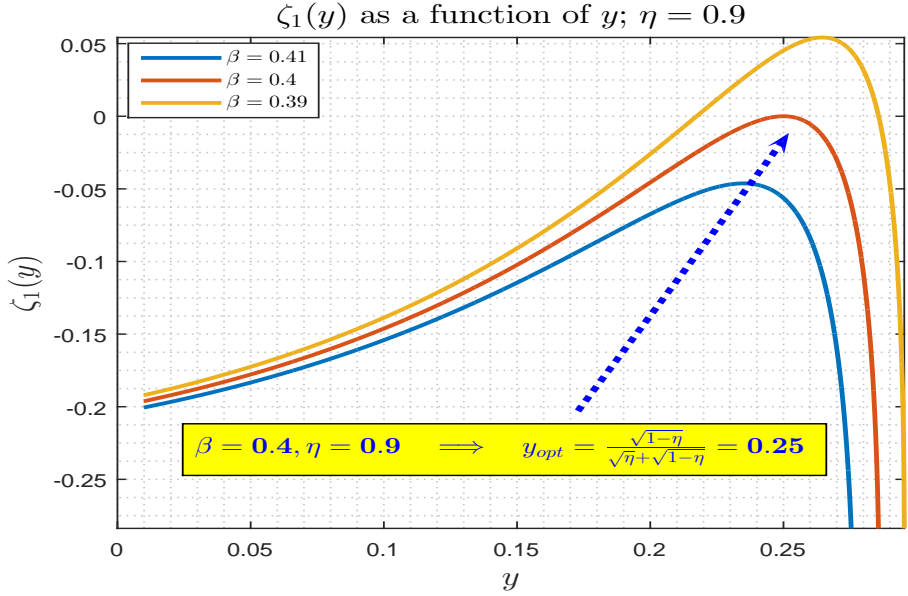,width=13.5cm,height=8cm}}
\caption{$\zeta_1{y}$ as a function of $y$}
\label{fig:zeta1ofy}
\end{figure}

We summarize the above results in the following lemma.

\begin{lemma}
  Assume the setup of Lemmas \ref{lemma:typwclemma1a} and \ref{lemma:typaclemma1} with ${\cal Q}_1$ as in (\ref{eq:typacanl0a}). Then we have for the upper edge of the ${\cal Q}_1$'s spectrum
\begin{equation}
 \beta = 1-2\sqrt{\eta(1-\eta)} \quad \Longleftrightarrow \quad \lambda_{max}({\cal Q}_1) = 1. \label{eq:typaclemma2eq1}
\end{equation}
Moreover,
\begin{equation}
\beta\leq 1-2\sqrt{\eta(1-\eta)} \quad \Longleftrightarrow \quad \lambda_{max}({\cal Q}_1)\leq 1. \label{eq:typaclemma2eq2}
 \end{equation}
   \label{lemma:typaclemma2}
\end{lemma}
\begin{proof}
  Follows from the above discussion.
\end{proof}

\subsubsection{The spectrum of ${\cal Q}_1 \triangleq \lambda_V^T\lambda_V\lambda_U^T\lambda_U$ -- practical evaluations}
\label{sec:fpteqvspecacprac}

Now that we have fully characterized the upper edge of the ${\cal Q}_1$'s spectrum we can return to the consideration of the entire spectrum. Relying on the above presented machinery we can establish the following lemma.

\begin{lemma}
  Assume the setup of Lemmas \ref{lemma:typwclemma1a} and \ref{lemma:typaclemma1} with ${\cal Q}_1$ as in (\ref{eq:typacanl0a}). Let $G_{{\cal Q}_1}(z)$ be the solution of the following system of equations:
{\small \begin{eqnarray}
  y(z) & = & \frac{\sqrt{zG_{{\cal Q}_1}(z)-1}}{\sqrt{zG_{{\cal Q}_1}(z)-1} + z\sqrt{G_{{\cal Q}_1}(z)}  } \nonumber \\
G_{\tilde{D}}(y(z)) & = & \frac{y(z)-(\beta+\eta) \pm \sqrt{(y(z)-(\beta+\eta))^2+4\beta\eta(y(z)-1)}}{2((y(z))^2-y(z))} \nonumber \\
  \frac{1}{\beta}  \lp G_{\tilde{D}}(y(z)) - \frac{\eta}{y(z)}  - \frac{1-(\beta+\eta)}{y(z)-1}\rp
  & = & - (\sqrt{zG_{{\cal Q}_1}(z)-1} + \sqrt{G_{{\cal Q}_1}(z)})(\sqrt{zG_{{\cal Q}_1}(z)-1}
+ z\sqrt{G_{{\cal Q}_1}(z)}).\nonumber \\ \label{eq:typaclemma3eq1}\end{eqnarray}}Then the spectral function of ${\cal Q}_1$, $f_{{\cal Q}_1}(x)$, is obtained as
\begin{eqnarray}
f_{{\cal Q}_1}(x)= - \lim_{\epsilon\rightarrow 0^+} \frac{\text{\mbox{\emph{imag}}}(G_{{\cal Q}_1}(x+i\epsilon))}{\pi}.
   \label{eq:typaclemma3eq2}
     \end{eqnarray}
   \label{lemma:typaclemma3}
\end{lemma}
\begin{proof}
  Follows from Lemma \ref{lemma:typaclemma1} through a combination of the results of  Lemma \ref{lemma:typwclemma1a} (where $G_{\tilde{D}}(\cdot)$ is explicitly given) and (\ref{eq:typacanl32})-(\ref{eq:typacanl34}). The following two sequences of identities are then sufficient to prove the lemma
  \begin{eqnarray}
  y(z) & = & \frac{z_1(z)+1}{z_1(z) S_{D}(z_1(z))} \nonumber \\
       & = & \frac{((-zG_{{\cal Q}_1}(z)- \sqrt{G_{{\cal Q}_1}(z)}\sqrt{zG_{{\cal Q}_1}(z)-1})+1) \sqrt{G_{{\cal Q}_1}(z)}}
       { (-zG_{{\cal Q}_1}(z)- \sqrt{G_{{\cal Q}_1}(z)}\sqrt{zG_{{\cal Q}_1}(z)-1} ) (\sqrt{zG_{{\cal Q}_1}(z)-1} + \sqrt{G_{{\cal Q}_1}(z)} )} \nonumber \\
 & = & \frac{\sqrt{zG_{{\cal Q}_1}(z)-1}}{\sqrt{zG_{{\cal Q}_1}(z)-1} + z\sqrt{G_{{\cal Q}_1}(z)}  },
  \label{eq:typaclemma3preq1}\end{eqnarray}
and
  \begin{eqnarray}
  \frac{z_1(z)+1}{y(z)} & = & \frac{S_{D}(z_1(z))}{z_1(z)} \nonumber \\
       & = &
       \frac{ (-zG_{{\cal Q}_1}(z)- \sqrt{G_{{\cal Q}_1}(z)}\sqrt{zG_{{\cal Q}_1}(z)-1} ) (\sqrt{zG_{{\cal Q}_1}(z)-1} + \sqrt{G_{{\cal Q}_1}(z)} )}
       {\sqrt{G_{{\cal Q}_1}(z)}} \nonumber \\
 & = & (\sqrt{zG_{{\cal Q}_1}(z)-1} + z\sqrt{G_{{\cal Q}_1}(z)} )(\sqrt{zG_{{\cal Q}_1}(z)-1} + \sqrt{G_{{\cal Q}_1}(z)} ).
  \label{eq:typaclemma3preq2}\end{eqnarray}
\end{proof}

In Figure \ref{fig:fcalQ1beta04eta09} we show the entire spectrum of $f_{{\cal Q}_1}(x)$. We chose $\beta=0.4$ and $\eta=0.9$ and ran the experiments with $n=4000$. As can be seen from the figure, the obtained numerical results are in a strong agreement with what the theory predicts.

\begin{figure}[htb]
\centering
\centerline{\epsfig{figure=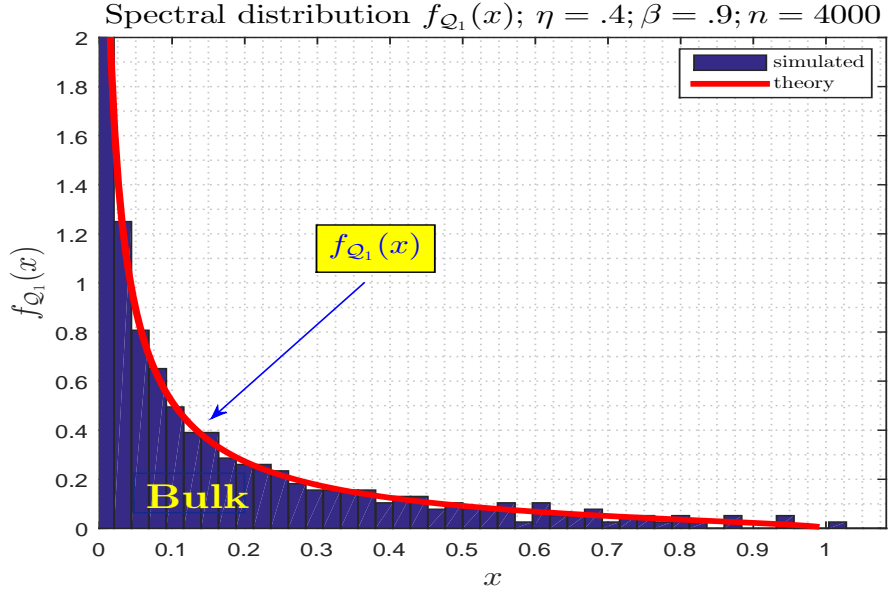,width=13.5cm,height=8cm}}
\caption{$f_{{\cal Q}_1}(x)$ -- spectral function of ${\cal Q}_1$; $\beta=0.4$ and $\eta=0.9$}
\label{fig:fcalQ1beta04eta09}
\end{figure}

\subsubsection{$\ell_0^*-\ell_1^*$-equivalence via the spectral limit -- asymmetric scenario}
\label{sec:fpteqvsplimac}

From Corollary \ref{cor:cinfcor3}, (\ref{eq:cinfcor3eq1}), and (\ref{eq:typacanl0a0}) one has in the \emph{asymmetric scenario}
 \begin{eqnarray}
\ell_0^*-\ell_1^*-\mbox{equivalence} \quad \Longleftrightarrow \quad \lambda_{max}(\lambda_V^T\lambda_V\lambda_U^T\lambda_U)\leq 1
\quad\Longleftrightarrow \quad \lambda_{max}({\cal Q}_1)\leq 1. \label{eq:typacanl77}
\end{eqnarray}
From (\ref{eq:typaclemma2eq2}) and (\ref{eq:typacanl77}) we finally have
\begin{equation}\label{eq:typacanl78}
  \ell_0^*-\ell_1^*-\mbox{equivalence} \quad\Longleftrightarrow \quad \beta  \leq  1 - 2\sqrt{\eta-\eta^2}.
 \end{equation}
Analogously to Theorem \ref{thm:typwcthm1} we can now establish  a precise \emph{asymmetric scenario} location of the phase transition in a typical statistical context.

\begin{theorem}(\textbf{\bl{$\ell_1^*$ -- phase transition -- C-inf (typical \prp{\underline{asymmetric scenario}})}})
   Assume the setup of Theorem \ref{thm:typwcthm1} with rank-$k$ matrix  $X_{sol}=X\in\mR^{n\times n}$ that now has Haar distributed \textbf{independent} bases of its orthogonal row and column spans $\bU^{\perp}\in\mR^{n\times (n-k)}$ and $\bV^{\perp}\in\mR^{n\times (n-k)}$ ($X_{sol}^T\bU^{\perp}=X_{sol}\bV^{\perp}=\0_{n\times (n-k)}$). Let $M\triangleq M^{(l)}\in\mR^{n\times n}$ be as defined in (\ref{eq:cinfanl2a}). Let $\beta_{ac}$ and $\eta$ satisfy the  following
\begin{center}
 \begin{tcolorbox}[beamer,title=\textbf{C-inf $\ell_1^*$ \yellow{asymmetric scenario} phase transition (PT)} characterization,lower separated=false, colback=yellow!95!green!40!white,
colframe=red!75!blue!60!black,fonttitle=\bfseries,width=5.5in]
  \begin{equation}\label{eq:typacthm1eq1}
    \xi_{\eta}^{(ac)}(\beta) \triangleq \beta-1+2\sqrt{\eta-\eta^2}=0.
  \end{equation}
 \end{tcolorbox}
\end{center}
 \noindent \textbf{If and only if} $\beta\leq \beta_{ac}$
    \begin{equation}\label{eq:typacthm1eq2}
   \lim_{n\rightarrow\infty} \mP(\ell_0^*\Longleftrightarrow \ell_1^*)=\ \lim_{n\rightarrow\infty} \mP(\mathbf{RMSE}=0)=1,
  \end{equation}
and the solutions of (\ref{eq:genmcl0posmmt}) and (\ref{eq:genmcl1posmmt}) coincide with overwhelming probability.
  \label{thm:typacthm1}
\end{theorem}
\begin{proof}
Follows from Lemma \ref{lemma:typaclemma2} and the above discussion.
\end{proof}

The results related to the use of the $\ell_1^*$-minimization heuristic for solving the causal inference problems obtained based on the above theorem are shown in Figure \ref{fig:cinfspectypacPTbetaeta}. As in the worst case scenario, the phase transition (PT) curve splits the $(\beta,\eta)$ region into two separate subregions where the $\ell_0^*-\ell_1^*$-equivalence phenomenon either occurs or fails to occur. Basically, below the curve one has a perfect recovery with the residual $\mathbf{RMSE}=\|\vecw(\hat{X})-\vecw(X_{sol})\|_2=0$. Contrary to that, above the curve though, there is an $X_{sol}$ for which $\mathbf{RMSE}\rightarrow\infty$ and $\ell_1^*$ fails.

\begin{figure}[htb]
\centering
\centerline{\epsfig{figure=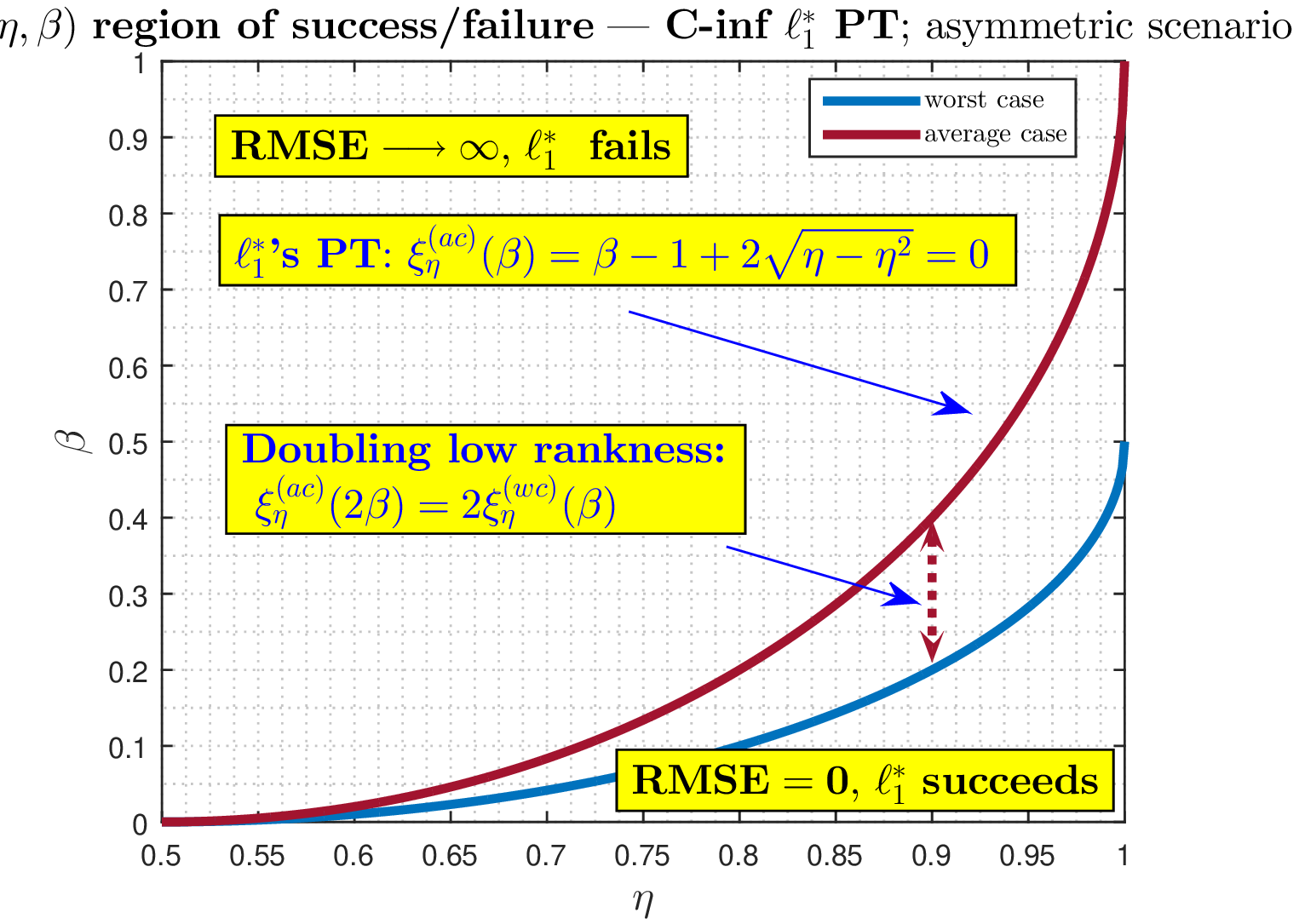,width=13.5cm,height=8cm}}
\caption{Causal inference (\bl{\textbf{C-inf}}) -- typical \emph{\textbf{asymmetric scenario}} $\ell_1^*$ phase transition}
\label{fig:cinfspectypacPTbetaeta}
\end{figure}

The following corollary adapts the above results so that they fit the standard $(\alpha,\beta)$ representation typically used in the compressed sensing (CS), low rank recovery (LRR), and matrix completion (MC) literature.

\begin{corollary}(\textbf{\bl{$\ell_1^*$ -- phase transition -- C-inf (typical \prp{\underline{asymmetric scenario}}; standard $(\alpha,\beta)$ representation)}})
   Assume the setup of Theorem \ref{thm:typacthm1}. Let $m$ be the total number of ones in matrix $M$ and let $\alpha\triangleq\lim_{n\rightarrow\infty}\frac{m}{n^2}$. Let $\beta$ and $\alpha_w$ satisfy the
\begin{center}
 \begin{tcolorbox}[beamer,title=\textbf{C-inf $\ell_1^*$ \yellow{asymmetric scenario} PT (standard $(\alpha,\beta)$ representation)},lower separated=false, colback=yellow!95!green!40!white,
colframe=red!75!blue!60!black,coltext=black,fonttitle=\bfseries,width=5in]
  \begin{equation}\label{eq:typaccor2eq1}
    \xi_{\beta}^{(wc,s)}(\alpha) \triangleq \beta-1+2\sqrt{\sqrt{1-\alpha}-1+\alpha}=0.
  \end{equation}
 \end{tcolorbox}
\end{center}
\noindent \textbf{If and only if} $\alpha\geq \alpha_{w}$
    \begin{equation}\label{eq:typaccor2eq2}
   \lim_{n\rightarrow\infty} \mP(\ell_0^*\Longleftrightarrow \ell_1^*)=\ \lim_{n\rightarrow\infty} \mP(\mathbf{RMSE}=0)=1,
  \end{equation}
and the solutions of (\ref{eq:genmcl0posmmt}) and (\ref{eq:genmcl1posmmt}) coincide with overwhelming probability.
  \label{cor:typaccor2}
\end{corollary}

\begin{proof}
  Follows as a direct consequence of Theorem \ref{thm:typacthm1} after noting that $m=n^2-(n-l)^2$ and consequently $\alpha=1-(1-\eta)^2$.
\end{proof}

 Figure \ref{fig:cinfspectypacPTalphabeta} shows the results obtained based on the above corollary in the standard $(\alpha,\beta)$ region format. As usual in the PT considerations, the entire $(\alpha,\beta)$ region is split in the part below the curve where $\mathbf{RMSE}=\|\vecw(\hat{X})-\vecw(X_{sol})\|_2=0$ and the part above the curve where even $\mathbf{RMSE}\rightarrow\infty$ is achievable.

We should point out an interesting similarity between what we observed here in the above corollary and in Figure \ref{fig:cinfspectypacPTalphabeta} on the one side and what is known to hold in generic LRR. Namely, as Corollary \ref{cor:typaccor2} states (and as is emphasized in  Figure \ref{fig:cinfspectypacPTalphabeta}), for the same value of $\alpha$ one achieves exactly two times larger $\beta$ in the asymmetric case than in the worst case. As the worst case is basically symmetric, one has that the PTs of the symmetric and the nonsymmetric scenarios are distinguished by a factor of two. Similar observation was in place when it comes to the comparison between the LRR of the symmetric and the general (nonsymmetric) matrices. However, one should keep in mind a fundamental difference as well. In LRR the underlying symmetry is \emph{a priori known} and can be utilized in the algorithms design whereas here it is just the choice of the worst case problem instance and is not assumed to be known to the algorithm itself. Of course, given the properties of the LRR, such a choice is not necessarily very surprising.

\begin{figure}[htb]
\centering
\centerline{\epsfig{figure=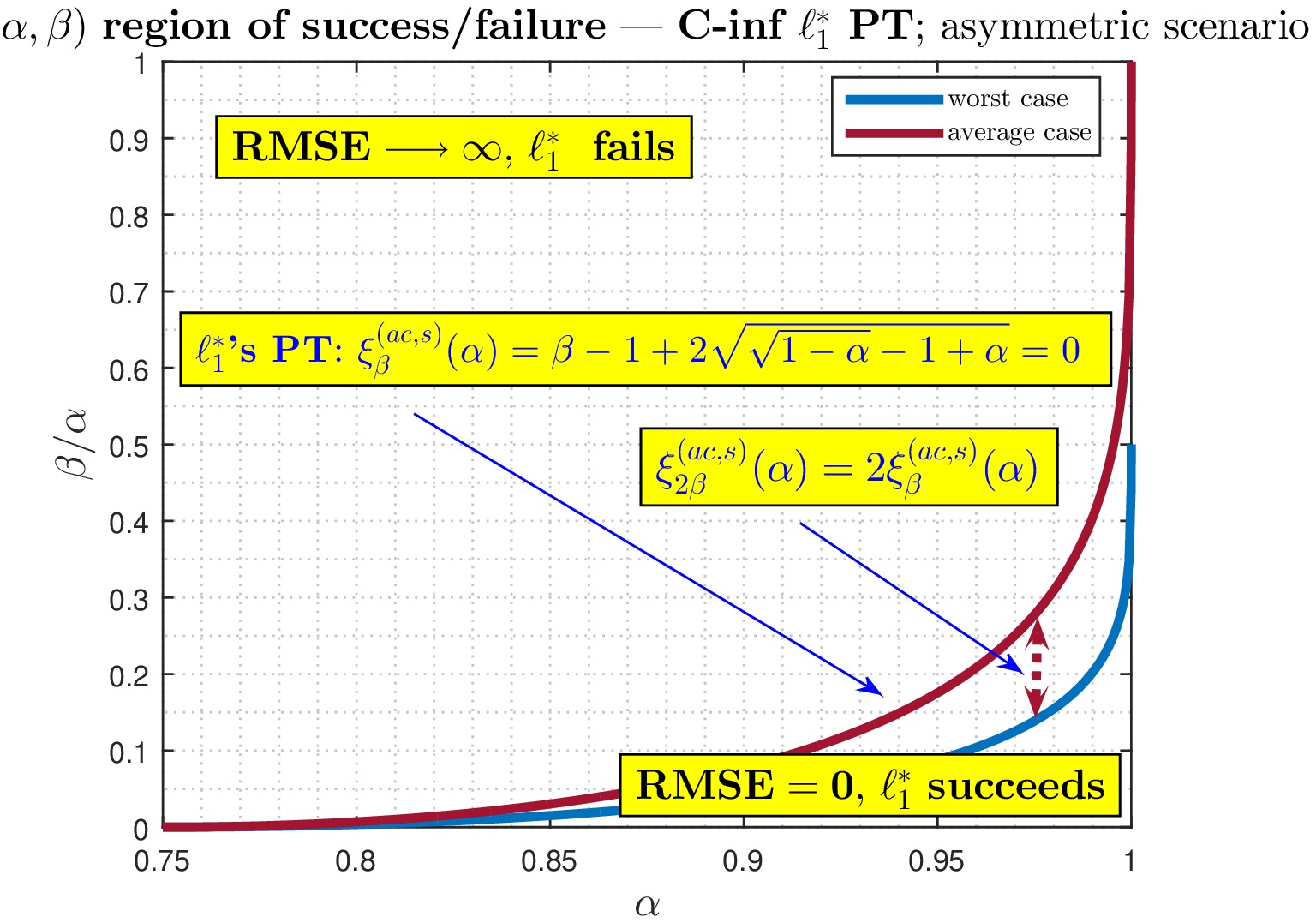,width=13.5cm,height=8cm}}
\caption{Causal inference (\bl{\textbf{C-inf}}) -- typical \emph{\textbf{asymmetric scenario}} $\ell_1^*$ phase transition ($(\alpha,\beta)$ region)}
\label{fig:cinfspectypacPTalphabeta}
\end{figure}

\subsection{Numerical results}
\label{sec:cpmcnumres}

To complement the above theoretical findings and see how successful in characterizing the utilization of the $\ell_1^*$-minimization in C-inf problems they indeed are, we conducted a set of numerical experiments and show the obtained results in Figure \ref{fig:numrestypacbetaetaptsim}.  As in \cite{Cinfidealwc22}, we again observe both the PT's existence and a solid agreement between its theoretical prediction and the results obtained through the simulations.

\begin{figure}[htb]
\centering
\centerline{\epsfig{figure=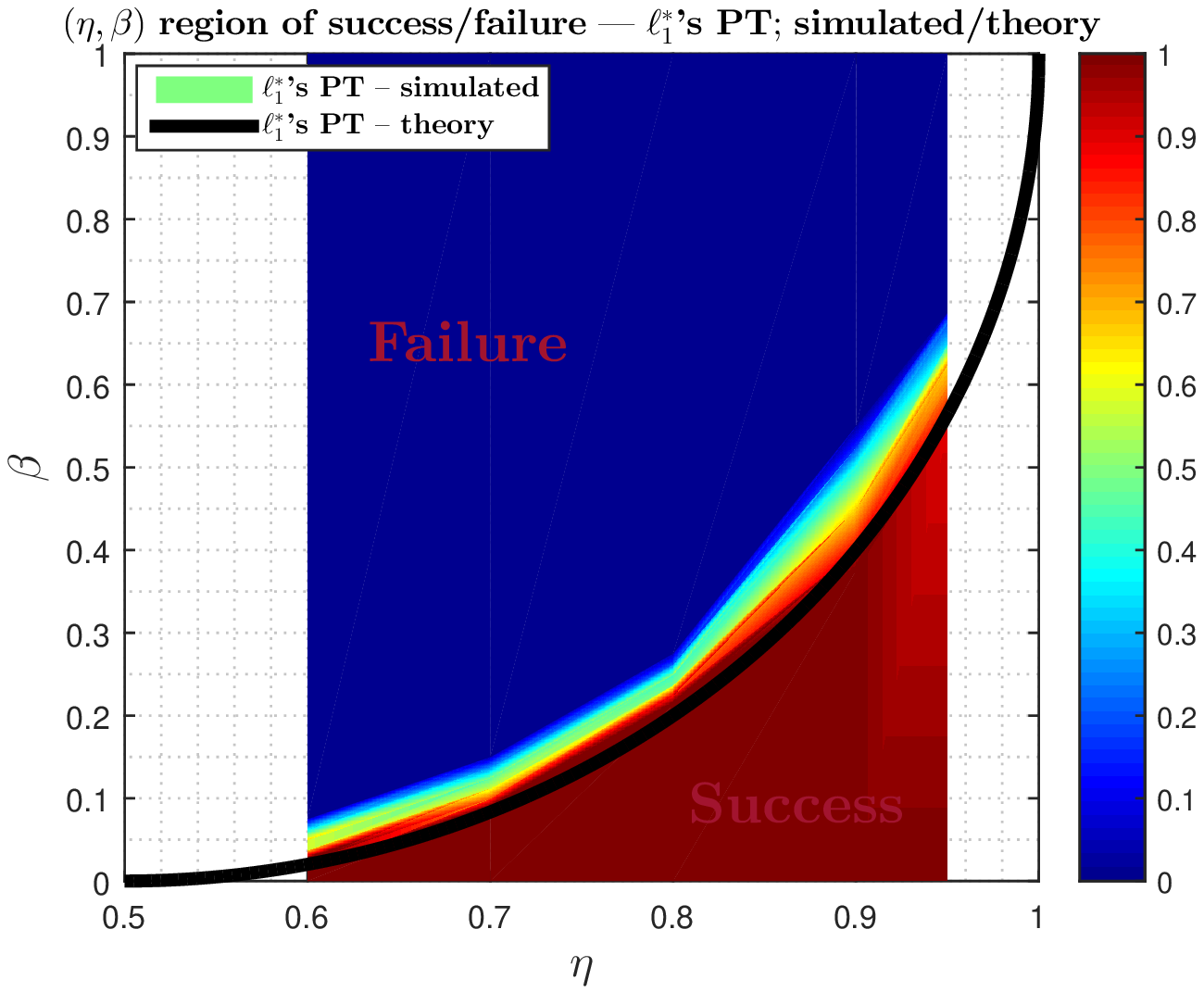,width=13.5cm,height=8cm}}
\caption{C-inf $\ell_1^{*}$'s asymmetric scenario phase transition (PT)}
\label{fig:numrestypacbetaetaptsim}
\end{figure}

In the conducted numerical experiments we chose $n=80$ and $\eta$ in the range $[0.6,0.95]$. Clearly, such fairly small matrix sizes correspond to the settings quite opposite from the ones that we used in the theoretical analysis. Still, even though the theory is predicated on the large $n$ assumption, it is not impossible that its conclusions remain valid for smaller values of $n$ as well. The results form Figure \ref{fig:numrestypacbetaetaptsim} confirm that this is indeed the case. Moreover, one can then say that the large $n$ regime, needed for the theoretical consideration, practically may start ro kick in already for rather small (of order of a few tens!) values of $n$. This ultimately means that the presented results, although theoretical in nature, have in themselves a strong practical component as well. Finally, we should also add that for larger values of $n$ an even better agreement between the theoretical and the simulated results is to be expected.

A few additional points regarding the simulations setup might be useful. First, one should emphasize, that in order to be in an agreement with the theoretical considerations, we, in all numerical experiments, considered the so-called \emph{typical} behavior. Following further into the footsteps of the theoretical considerations, the presented simulations results were obtained for the square matrices. As was the case in \cite{Cinfidealwc22}, all theoretical considerations can be repeated assuming the non-square scenarios as well. We, however, (as in \cite{Cinfidealwc22}) prioritized the clarity of the presentations over simple generalizations and opted for the square scenarios which are substantially easier to present. Also, all the simulations needed for Figure \ref{fig:numrestypacbetaetaptsim} were done with the singular values of the unknown targeted  matrices equal to one. While we refer to \cite{Cinfidealwc22} for a bit more complete discussion regarding such a choice,
we here briefly mention that choices of this type are known to serve as the worst case examples in establishing the reversal $\ell_0-\ell_1$-equivalence conditions. As in \cite{Cinfidealwc22}, we also ran the simulations where the singular values were randomly chosen with results either identically matching or improving on the ones shown in Figure \ref{fig:numrestypacbetaetaptsim}.

\section{Conclusion}
\label{sec:conc}

In this paper, we have built on  the mathematical \prp{\textbf{\emph{Causal inference (C-inf) $\leftrightarrow$  \textbf{low-rank recovery (LRR)}}}} connection established  in \cite{Cinfidealwc22} to deal with asymmetric PTs phenomena. The results of \cite{Cinfidealwc22} proved that the nuclear norm ($\ell_1^*$) minimization heuristic, when used for solving the low rank recovery C-inf problems, exhibits the so-called phase transition (PT) phenomenon. Moreover, in a typical statistical scenario, \cite{Cinfidealwc22} characterized the exact location of the \emph{worst case} PT. This effectively meant that there are problem instances where the $\ell_1^*$  predicated behavior might be improved upon. Here we showed that this is indeed true. Considering an \emph{asymmetric scenario} (in contrast with the symmetric worst case one from \cite{Cinfidealwc22}) we determined the underlying exact phase transitions locations. Moreover, we uncovered a \prp{\textbf{\emph{doubling low rankness}}} phenomenon, which means that, throughout the entire region of allowed system parameters, matrices of exactly two times larger rank can be recovered when compared to the worst case scenario from \cite{Cinfidealwc22}. Such a phenomenon also ensures that the simplicity of the worst case PTs from \cite{Cinfidealwc22} is preserved in the asymmetric scenarios as well. Consequently, one is again able to elegantly pin down the relation between the low rankness of the target C-inf matrix and the time when the treatment is applied.

Throughout the process of creating the theoretical phase transitions characterizations we also established several mathematical results that are of independent interest. All of our theoretical findings we supplemented with the results obtained from the corresponding numerical experiments. Moreover, in all cases we observed a rather overwhelming agreement between what the theory predicts and what the simulations provide.

To achieve the desired phase transition results we relied on a combination of the ideas from the Random duality theory (\bl{\textbf{RDT}}) and the Free probability theory (\bl{\textbf{FPT}}). As a result, we obtained a very powerful and generic mathematical apparatus that will serve as a theoretical platform in further explorations. As this and the companion paper \cite{Cinfidealwc22} are of an introductory nature we stopped short of showcasing how the created theory fairs when utilized for handling more complex problem instances (these among others include the, practically very relevant, noisy and approximately low-rank corresponding ones). Our companion papers will establish results along these directions relying on the mathematical framework presented here and in \cite{Cinfidealwc22}.

\begin{singlespace}
\bibliographystyle{plain}
\bibliography{cinfspecidealRefs1}
\end{singlespace}

\appendix

\section{Proof of Theorem \ref{thm:cinfthm1}}
\label{sec:appA}

As mentioned earlier, the proof of Theorem \ref{thm:cinfthm1} is conceptually identical to the corresponding proof when matrix $X$ is symmetric. A detailed proof for the symmetric matrices is given below. Before being able to present the proof we need a couple of technical lemmas.

\begin{lemma}
Let $C=C^T\in\mR^{n\times n}$. Also let all eigenvalues of $C$ belong to the interval $[-1,1]$. Finally, let the first $k$ entries on the main diagonal, $C_{i,i},1\leq i\leq k$, be larger than or equal to 1. Then the upper $k\times k$  left block of $C$, $C_{1:k,1:k}$, is an identity matrix, i.e.
  \begin{eqnarray}
     C_{1:k,1:k} & = & I_{k\times k}.\label{eq:genmcproof8}
  \end{eqnarray} \label{lemma:genmclemma3}
\end{lemma}
\begin{proof}
  Let $\lambda_{max}(C)$ be the maximum eigenvalue of $C$. Then
  \begin{eqnarray}
     \lambda_{max}(C) & \triangleq & \max_{\|\c\|_2=1}\c^TC\c.\label{eq:genmcproof9}
  \end{eqnarray}
Since by assumptions $1\leq  C_{i,i},1\leq i\leq k$ and $\lambda_{max}(C)\leq 1$ we also have for any $1\leq i\leq k$
  \begin{eqnarray}
   1\leq  C_{i,i} \leq\max_{\|\c\|_2=1} \c^TC\c\triangleq \lambda_{max}(C) \leq 1,\label{eq:genmcproof10}
  \end{eqnarray}
which implies $C(i,i)=1,1\leq i\leq k$. The proof that all other elements of $C_{1:k,1:k}$ are equal to zero proceeds inductively.

\underline{\textbf{1) Induction move from $l=1$ to $l=2$:}}
First we look at the upper block of size $2\times 2$, i.e. at $C_{1:2,1:2}$. We then have
  \begin{eqnarray}
     1 \geq \max_{\|\c\|_2=1}\c^TC\c \geq \max_{\|\c_{1:2}\|_2=1} \c_{1:2}^TC_{1:2,1:2}\c_{1:2}
     & \geq  & \max_{\|\c_{1:2}\|_2=1} \lp \|\c_{1:2}\|_2+2|\c_1\c_2C_{1,2}|\rp \nonumber \\
     & \geq  & \max_{\|\c_{1:2}\|_2=1} \lp 1+2|\c_1\c_2C_{1,2}|\rp \geq 1,\label{eq:genmcproof11}
  \end{eqnarray}
which implies $C_{1,2}=0$.

\underline{\textbf{2) Induction move from $l$ to $l+1$:}}
Now we look at the upper block of size $(l+1)\times (l+1)$, i.e. at $C_{1:l+1,1:l+1}$ while assuming that $C_{1:l,1:l}=I_{l\times l}$. We then have
  \begin{eqnarray}
     1 & \geq  & \max_{\|\c\|_2=1}\c^TC\c \nonumber \\
      & \geq & \max_{\|\c_{1:l+1}\|_2=1}\c_{1:l+1}^TC_{1:l+1,1:l+1}\c_{1:l+1} \nonumber \\
      & \geq & \max_{\|\c_{1:l+1}\|_2=1} \lp \|\c_{1:l+1}\|_2+2|\c_{1:l}^TC_{1:l,l+1}\c_{l+1}| \rp \nonumber \\
      & \geq & \max_{\|\c_{1:l+1}\|_2=1} \lp 1+2|\c_{1:l}^TC_{1:l,l+1}\c_{l+1}| \rp \nonumber \\
      & \geq & 1,\label{eq:genmcproof12}
  \end{eqnarray}
which implies $C_{1:l,l+1}=\0_{l\times 1}$ and completes the proof.
\end{proof}

\begin{lemma}
Assume the setup of Lemma \ref{lemma:genmclemma3}. Then the upper $k\times k$  left block of $C$, $C_{1:k,1:k}$, is an identity matrix and the upper $k\times (n-k)$  right  block of $C$, $C_{1:k,n-k+1:n}$ is a zero matrix, i.e.
  \begin{eqnarray}
     C_{1:k,1:k} & = & I_{k\times k} \nonumber \\
     C_{1:k,n-k+1:n} & = & \0_{k\times (n-k)}.\label{eq:genmcproof18}
  \end{eqnarray} \label{lemma:genmclemma4}
\end{lemma}
\begin{proof}
The first part follows by Lemma \ref{lemma:genmclemma3}. We now focus on the second part. Consider the following partition of matrix $C$
  \begin{eqnarray}
     C & = & \begin{bmatrix}
               C_{1:k,1:k} & C_{1:k,n-k+1:n} \\
               C_{n-k+1:n,1:k} & C_{n-k+1:n,n-k+1:n}
             \end{bmatrix} =\begin{bmatrix}
               I_{k\times k} & C_{1:k,n-k+1:n} \\
               C_{n-k+1:n,1:k} & C_{n-k+1:n,n-k+1:n}
             \end{bmatrix}.\label{eq:genmcproof14}
  \end{eqnarray}
Then assuming that the largest nonzero singular value of $C_{1:k,n-k+1:n}$ is equal to $b>0$, we have
{\small   \begin{eqnarray}
     1 & \geq  & \max_{\|\c\|_2=1}\c^TC\c \nonumber \\
      & \geq & \max_{\|\c_{1:k}\|_2=a,\c_{n-k+1:n}} \lp \c_{1:k}^TC_{1:k,1:k}\c_{1:k} +2|\c_{1:k}^TC_{1:k,n-k+1:n}\c_{n-k+1:n}| + \c_{n-k+1:n}^TC_{n-k+1:n,n-k+1:n}\c_{n-k+1:n}\rp\nonumber \\
      & \geq & \max_{\|\c_{1:k}\|_2=a,\c_{n-k+1:n}} \lp a^2 +2|\c_{1:k}^TC_{1:k,n-k+1:n}\c_{n-k+1:n}| + \c_{n-k+1:n}^TC_{n-k+1:n,n-k+1:n}\c_{n-k+1:n}\rp\nonumber \\
      & \geq & \max_{\|\c_{1:k}\|_2=a,\c_{n-k+1:n}} \lp a^2 +2|\c_{1:k}^TC_{1:k,n-k+1:n}\c_{n-k+1:n}| - \c_{n-k+1:n}^T\c_{n-k+1:n}\rp \nonumber \\
      & \geq & \max_{a\in[0,1]} \lp a^2 +2ba\sqrt{1-a^2} - (1-a^2) \rp\nonumber \\
            & = & \max_{a\in[0,1]} \lp 2a^2-1 +2ba\sqrt{1-a^2} \rp,\label{eq:genmcproof15}
  \end{eqnarray}}where the fourth inequality follows since the minimum eigenvalue of $C_{n-k+1:n,n-k+1:n}$ is larger than or equal to the minimum eigenvalue of $C$ which is by the lemma's assumption larger than or equal to -1. Now, we further have
  \begin{eqnarray}
  c\triangleq 2a\sqrt{1-a^2} \quad \mbox{and} \quad 2a^2-1 +2ba\sqrt{1-a^2}=\sqrt{1-c^2}+b c,\label{eq:genmcproof16}
  \end{eqnarray}
and
  \begin{eqnarray}
  \frac{d(\sqrt{1-c^2}+b c)}{dc}=\frac{-c}{\sqrt{1-c^2}}+b=0.\label{eq:genmcproof17}
  \end{eqnarray}
From (\ref{eq:genmcproof17}) we then easily obtain
  \begin{eqnarray}
c=\frac{b}{\sqrt{1+b^2}}.\label{eq:genmcproof18}
  \end{eqnarray}
A combination of (\ref{eq:genmcproof15}), (\ref{eq:genmcproof16}), and (\ref{eq:genmcproof18}) gives
  \begin{eqnarray}
     1  \geq   \max_{\|\c\|_2=1}\c^TC\c \geq \max_{a\in[0,1]} \lp 2a^2-1 +2ba\sqrt{1-a^2}\rp  = &\sqrt{1+b^2},\label{eq:genmcproof19}
  \end{eqnarray}
which implies $b=0$ and automatically $C_{1:k,n-k+1:n}=\0_{k\times 1}$. This completes the proof.
\end{proof}

Now we can consider the above mentioned theorem that adapts the general $\ell_1$ equivalence condition result from \cite{StojnicCSetam09,StojnicUpper10,StojnicICASSP09} to the corresponding one for the $\ell_1$ norm of the singular/eigenvalues (similar adaptation can also be found in \cite{OH10}).

\begin{theorem}(\textbf{\bl{$\ell_0^*-\ell_1^*$-equivalence condition (LRR)}} -- \textbf{symmetric}  $X$)
 Consider a $\bU\in\mR^{n\times k}$ such that $\bU^T\bU=I_{k\times k}$ and a $\rankw-k$ \textbf{a priori known to be} symmetric matrix $X_{sol}=X\in\mR^{n\times n}$  with all of its columns belonging to the span of $\bU$. For concreteness, and without loss of generality, assume that $X$ has only positive nonzero eigenvalues. For a given matrix $A\in\mR^{m\times n^2}$ ($m\leq n^2$) assume that $\y=A\vecw(X)\in \mR^m$. If
\begin{equation}
(\forall W\in \mR^{n\times n} | A\vecw(W)=\0_{m\times 1},W=W^T\neq \0_{n\times n}) \quad  -\tr(\bU^TW\bU)< \ell_1^*((\bU^{\perp})^TW\bU^{\perp}),
\label{eq:genmcposthmcond1}
\end{equation}
then the solutions of (\ref{eq:genmcl0posmmt}) and (\ref{eq:genmcl1posmmt}) coincide. Moreover, if
\begin{equation}
(\exists  W\in \mR^{n\times n} | A\vecw(W)=\0_{m\times 1},W=W^T\neq \0_{n\times n}) \quad  -\tr(\bU^TW\bU)\geq \ell_1^*((\bU^{\perp})^TW\bU^{\perp}),
\label{eq:genmcposthmcond2}
\end{equation}
then there is an $X$ from the above set of the symmetric matrices with columns belonging to the span of $\bU$  such that the solutions of (\ref{eq:genmcl0posmmt}) and (\ref{eq:genmcl1posmmt}) are different.
\label{thm:genmcthmregposcond}
\end{theorem}
\begin{proof}
The proof follows literally step-by-step the proof of the corresponding theorem in \cite{StojnicCSetam09,StojnicICASSP09,StojnicUpper10} and adapts it to matrices or their singular/eigenvalues. For experts in the field this adaptation is highly likely to be viewed as trivial and certainly doesn't need to be as detailed as we will make it to be. Nonetheless, to ensure a perfect clarity of all arguments we provide a step-by-step instructional derivation. For concreteness and without loss of generality we also assume that the eigen-decomposition of $X$ is

\begin{eqnarray}\label{eq:genmcrec1}
X=U\Lambda U^T=\begin{bmatrix}
                \bU & \bU^{\perp}
              \end{bmatrix}
              \begin{bmatrix}
                \bar{\Lambda}_X & \0_{k\times (n-k)} \\
                \0_{(n-k)\times k} & \bar{\Lambda}_X^{\perp}
              \end{bmatrix}
              \begin{bmatrix}
                \bU & \bU^{\perp}
              \end{bmatrix}^T.
\end{eqnarray}

\bl{ \underline{\textbf{(i) $\Longrightarrow$ (the if part):}}} Following step-by-step the proof of Theorem $2$ in \cite{StojnicICASSP09}, we start by  assuming that  $\hat{X}$ is the solution of (\ref{eq:genmcl1posmmt}). Then we want to show that if (\ref{eq:genmcposthmcond1}) holds then $\hat{X}=X$. As usual, we instead of that, assume opposite, i.e. we assume that (\ref{eq:genmcposthmcond1}) holds but $\hat{X}\neq X$. Then since $\y=A\vecw(\hat{X})$ and $\y=A\vecw(X)$ must hold simultaneously there must exist $W$ such that $\hat{X} =X+W$ with $W\neq 0$, $A\vecw(W)=0$. Moreover, since $\hat{X}$ is the solution of (\ref{eq:genmcl1posmmt}) one must also have
\begin{eqnarray}
\begin{array}{r r r l@{\ }}
   & \ell_1^*(X+W) = \ell_1^*(\hat{X}) & \leq & \ell_1^*(X) \\
\Longleftrightarrow \hspace{.3in} $ $ & \ell_1^*(\begin{bmatrix}
                \bU & \bU^{\perp}
              \end{bmatrix}^T(X+W)\begin{bmatrix}
                \bU & \bU^{\perp}
              \end{bmatrix}) & \leq & \ell_1^*(X) \\
\Longrightarrow \hspace{.3in} $ $ &\ell_1^*(\bU^T (X+W)\bU)+\ell_1^*((\bU^{\perp})^T (X+W)\bU^{\perp}) & \leq & \ell_1^*(X).
\end{array}\nonumber \\
\label{eq:genmcproof1}
\end{eqnarray}
The last implication follows after one trivially notes
\begin{eqnarray}
 \ell_1^*(\begin{bmatrix}
                \bU & \bU^{\perp}
              \end{bmatrix}^T(X+W)\begin{bmatrix}
                \bU & \bU^{\perp}
              \end{bmatrix}) & = & \max_{\Lambda_*=\Lambda_*^T\in {\cal L}_*}
                            \tr(\Lambda_*
              \begin{bmatrix}
                \bU & \bU^{\perp}
              \end{bmatrix}^T(X+W)\begin{bmatrix}
                \bU & \bU^{\perp}
              \end{bmatrix})\nonumber \\
  & \geq & \max_{\Lambda_*=\Lambda_*^T\in {\cal L}_{*}^0}
              \tr(\Lambda_*
              \begin{bmatrix}
                \bU & \bU^{\perp}
              \end{bmatrix}^T(X+W)\begin{bmatrix}
                \bU & \bU^{\perp}
              \end{bmatrix})\nonumber \\
& = & \ell_1^*(\bU^T (X+W)\bU)+\ell_1^*((\bU^{\perp})^T (X+W)\bU^{\perp}),\label{eq:genmcproof1a}
\end{eqnarray}
where
\begin{eqnarray}
{\cal L}_{*}^0 &\triangleq& \left \{\Lambda_*\in\mR^{n\times n} | \Lambda_*=\Lambda_*^T,\Lambda_*\Lambda_*^T\leq I, \Lambda_*=\begin{bmatrix}
                                                                                                                 \Lambda_{*,1} & 0_{k\times (n-k)} \\
                                                                                                                 0_{(n-k)\times k} & \Lambda_{*,2}
                                                                                                               \end{bmatrix} \right \} \nonumber \\
   &\subseteq& \left \{\Lambda_*\in\mR^{n\times n} | \Lambda_*=\Lambda_*^T,\Lambda_*\Lambda_*^T\leq I\right \} \triangleq   {\cal L}_{*}.\label{eq:genmcproof1b}
\end{eqnarray}

\tcbset{beamer,lower separated=false, fonttitle=\bfseries,width=3.4in, coltext=white,
colback=yellow!70!orange!40!white,title style={left color=cyan!40!black!80!purple, right color=red!60!yellow!40!orange!80!white},
width=(\linewidth-4pt)/4, equal height group=AT,before=,after=\hfill,fonttitle=\bfseries}
\tcbset{colback=red!25!white!70!green!15!yellow,colframe=red!95!white,width=(\linewidth-4pt)/4,
equal height group=AT,before=,after=\hfill,fonttitle=\bfseries,
interior style={left color=cyan!40!black!80!purple, right color=red!60!yellow!40!orange!80!white}}
\noindent\begin{tcolorbox}[width=4.3in, height=.25in]
\vspace{-.27in}
\begin{equation*}
\hspace{-.0in}\mbox{\textbf{The key observation -- \bl{``\emph{Removing the absolute values}"}:}}
\vspace{-0in}
\end{equation*}
\end{tcolorbox}


\noindent Now, the key observation made in \cite{StojnicICASSP09} comes into play. Namely, one notes that the absolute values can be removed in the nonzero part and that the $\ell_1^*(\cdot)$ can be ``\emph{replaced}" by $\tr(\cdot)$.  Such a simple observation is the most fundamental reason for all the success of the \bl{\textbf{RDT}} when used for the \textbf{exact} performance characterization of the structured objects' recovery. From (\ref{eq:genmcproof1}) we then have
\begin{eqnarray}
\begin{array}{r r r l@{\ }}
 & \ell_1^*(\bU^T (X+W)\bU)+\ell_1^*((\bU^{\perp})^T (X+W)\bU^{\perp}) & \leq & \ell_1^*(X)\\
 \Longrightarrow   \hspace{.3in} $ $ & \tr(\bU^T (X+W)\bU)+\ell_1^*((\bU^{\perp})^T (W)\bU^{\perp}) & \leq & \ell_1^*(X)\\
 \Longleftrightarrow   \hspace{.3in} $ $ & \tr(\bU^T W\bU)+\ell_1^*((\bU^{\perp})^T W\bU^{\perp}) & \leq & 0.
\end{array}\label{eq:genmcproof2}
\end{eqnarray}
We have arrived at a contradiction as the last inequality in (\ref{eq:genmcproof2}) is exactly the opposite of (\ref{eq:genmcposthmcond1}). This implies that our initial assumption $\hat{X}\neq X$ cannot hold and we therefore must have $\hat{X}=X$. This is precisely the claim of the first part of the theorem.

\bl{ \underline{ \textbf{ (ii) $\Longleftarrow$ (the only if part):}}}  We now assume that (\ref{eq:genmcposthmcond2}) holds, i.e.
\begin{equation}
(\exists  W\in \mR^{n\times n} | A\vecw(W)=\0_{m\times 1},W\neq \0_{n\times n}) \quad  -\tr((\bU)^TW\bU)\geq \ell_1^*((\bU^{\perp})^TW\bU^{\perp})\label{eq:genmcproof3}
\end{equation}
and would like to show that for such a $W$ there is a symmetric rank-$k$ matrix $X$ with the columns belonging to the span of $\bU$ such that $\y=A\vecw(X)$,  and the following holds
\begin{equation}
\ell_1^*(X+W)<\ell_1^*(X).\label{eq:genmcproof4}
\end{equation}

Existence of such an $X$ would ensure that it both, satisfies all the constraints in (\ref{eq:genmcl1posmmt}) and is not the solution of (\ref{eq:genmcl1posmmt}). Following the strategy of  \cite{StojnicUpper10} one can reverse all the above steps from (\ref{eq:genmcproof3}) to (\ref{eq:genmcproof1}) with strict inequalities and arrive at the first inequality in (\ref{eq:genmcproof1}) which is exactly (\ref{eq:genmcproof4}). There are two implications that cause problems in such a reversal process, the one in (\ref{eq:genmcproof3}) and the one in (\ref{eq:genmcproof1}). If these implications were equivalences everything would be fine. We address these two implications separately.

\underline{1) \textbf{the implication in (\ref{eq:genmcproof2}) -- particular $X$ to ``overwhelm" $W$:}} Assume $X=\bU\Lambda_x\bU^T$ with $\Lambda_x>0$ being a diagonal matrix with arbitrarily large elements on the main diagonal (here it is sufficient even to choose diagonal of $\Lambda_x$ so that its smallest element is larger than the maximum eigenvalue of $\bU^TW\bU$). Now one of course sees the main idea behind the ``removing the absolute values" concept from \cite{StojnicICASSP09,StojnicUpper10}. Namely, for such an $X$ one has that $\ell_1^*(\bU^TX+W)\bU)=tr(\ell_1^*(\bU^TX+W)\bU))$ since for symmetric matrices  the $\ell_1^*(\cdot)$ (as the sum of the argument's \emph{absolute} eigenvalues) and $\tr(\cdot)$ (as the sum of the argument's eigenvalues) are equal. That basically means that when going backwards the second inequality in (\ref{eq:genmcproof2}) not only follows from the first one but also implies it as well. In other words, for $X=\bU\Lambda_x\bU^T$ (with $\Lambda_x>0$ and arbitrarily large)
\begin{eqnarray}
\begin{array}{r r r l@{\ }}
   & \tr(\bU^T W\bU)+\ell_1^*((\bU^{\perp})^T W\bU^{\perp}) & \leq & 0 \\
 \Longleftrightarrow  \hspace{.3in}  $ $ & \tr(\bU^T (X+W)\bU^)+\ell_1^*((\bU^{\perp})^T (W)\bU^{\perp}) & \leq & \ell_1^*(X)\\
 \Longleftrightarrow \hspace{.3in} $ $ & \ell_1^*(\bU^T (X+W)\bU)+\ell_1^*((\bU^{\perp})^T (X+W)\bU^{\perp}) & \leq & \ell_1^*(X),
\end{array}\label{eq:genmcproof5}
\end{eqnarray}
which basically mans that there is an $X$ that can ``overwhelm" $W$ (in the span of $\bU$) and ensures that the \bl{``\textbf{\emph{removing the absolute values}}"} is not only a \textbf{\emph{sufficient}} but also a \textbf{\emph{necessary}} concept for creating the relaxation  equivalence condition.

\underline{2) \textbf{the implication in (\ref{eq:genmcproof1}):}} One would now need to somehow show that the third inequality in (\ref{eq:genmcproof1}) not only follows from the second one but also implies it as well. This boils down to showing that inequality in (\ref{eq:genmcproof1a}) can be replaced with an equality or, alternatively, that ${\cal L}^0$ and ${\cal L}$ are provisionally equivalent. Neither of these statements is generically true. However, since we have a set of $X$ at our disposal there might be an $X$ for which they actually hold. We continue to assume $X=\bU\Lambda_x\bU^T$ with $\Lambda_x>0$ being a diagonal matrix with arbitrarily large entries on the main diagonal. Then the last equality in (\ref{eq:genmcproof1a}) gives
\begin{eqnarray}
\begin{array}{r r r l@{\ }}
  $ $ & \ell_1^*(\bU^T (X+W)\bU)+\ell_1^*((\bU^{\perp})^T (X+W)\bU^{\perp}) & \leq & \ell_1^*(X) \\
  \Longleftrightarrow \hspace{.3in} $ $ & \max_{\Lambda_*=\Lambda_*^T\in {\cal L}_{*}^0}
                            \tr(\Lambda_*
              \begin{bmatrix}
                \bU & \bU^{\perp}
              \end{bmatrix}^T(X+W)\begin{bmatrix}
                \bU & \bU^{\perp}
              \end{bmatrix}) & \leq & \ell_1^*(X).
\end{array}\label{eq:genmcproof6}
\end{eqnarray}
Also, one has
\begin{eqnarray}
\begin{array}{r r r l@{\ }}
  & \max_{\Lambda_*=\Lambda_*^T\in {\cal L}_{*}^0}
                            \tr(\Lambda_*
              \begin{bmatrix}
                \bU & \bU^{\perp}
              \end{bmatrix}^T(X+W)\begin{bmatrix}
                \bU & \bU^{\perp}
              \end{bmatrix}) & \leq & \ell_1^*(X)\\
                 \Longleftrightarrow \hspace{.3in} $ $ & \max_{\Lambda_{*,i}=\Lambda_{*,i}^T,\Lambda_{*,i}\Lambda_{*,i}^T\leq I,i\in\{1,2\}}
                            \tr(\Lambda_{*,1}\bU^T X\bU +\Lambda_{*,2}(\bU^{\perp})^T W\bU^{\perp}) & \leq & \ell_1^*(X) \\
                                             \Longleftrightarrow \hspace{.3in} $ $ & \max_{\Lambda_{*,i}=\Lambda_{*,i}^T,\Lambda_{*,i}\Lambda_{*,i}^T\leq I,i\in\{1,2\}}
                            \tr(\Lambda_{*,1}\Lambda_x +\Lambda_{*,2}(\bU^{\perp})^T W\bU^{\perp}) & \leq & \tr(\Lambda_x).
\end{array}\label{eq:genmcproof6}
\end{eqnarray}
Now, if at least one of the elements on the main diagonal of $\Lambda_{*,1}$, $\diag(\Lambda_{*,1})$, is smaller than 1, then the corresponding element on the diagonal of $\Lambda_x$ can be made arbitrarily large compared to the other elements of $\Lambda_x$ and one would have
\begin{eqnarray}
\begin{array}{r r r l@{\ }}
  & \max_{\Lambda_{*,i}=\Lambda_{*,i}^T,\Lambda_{*,i}\Lambda_{*,i}^T\leq I,i\in\{1,2\}}
                            \tr(\Lambda_{*,1}\Lambda_x +\Lambda_{*,2}(\bU^{\perp})^T W\bU^{\perp}) & < & \tr(\Lambda_x) \\
                            \Longleftrightarrow \hspace{.3in} $ $   & \max_{\Lambda_*=\Lambda_*^T\in {\cal L}_{*}^0}
                            \tr(\Lambda_*
              \begin{bmatrix}
                \bU & \bU^{\perp}
              \end{bmatrix}^T(X+W)\begin{bmatrix}
                \bU & \bU^{\perp}
              \end{bmatrix}) & < & \ell_1^*(X)\\
                                          \Longleftrightarrow \hspace{.3in} $ $   & \max_{\Lambda_*=\Lambda_*^T\in {\cal L}_{*}}
                            \tr(\Lambda_*
              \begin{bmatrix}
                \bU & \bU^{\perp}
              \end{bmatrix}^T(X+W)\begin{bmatrix}
                \bU & \bU^{\perp}
              \end{bmatrix}) & < & \ell_1^*(X),
\end{array}\label{eq:genmcproof7}
\end{eqnarray}
where the last equivalence holds since the difference of the terms on the left-hand side in the last two inequalities is bounded independently of $X$. Also, the last inequality in (\ref{eq:genmcproof7}) together with the first equality in (\ref{eq:genmcproof1a}) and the first inequality in (\ref{eq:genmcproof1}) produces (\ref{eq:genmcproof4}). Therefore the only scenario that is left as potentially not producing (\ref{eq:genmcproof4}) is when all the elements on the main diagonal are larger than or equal to 1. However, the two lemmas preceding the theorem show that in such a scenario ${\cal L}^0={\cal L}$ and one consequently has an equality instead of the inequality in (\ref{eq:genmcproof1a}) which then, together with
(\ref{eq:genmcproof1}), implies (\ref{eq:genmcproof4}). This completes the proof of the second (``the only if") part of the theorem and therefore of the entire theorem.
 \end{proof}

\end{document}